\documentclass{article}
\usepackage[utf8]{inputenc}
\usepackage{fullpage}
\usepackage{amsmath,amsfonts,amsthm,commath,amssymb,dsfont,mathtools,multicol}
\usepackage[normalem]{ulem}
\usepackage{graphicx}
\usepackage{grffile}
\usepackage{xcolor}
\usepackage{subfigure}
\usepackage{amssymb}
\usepackage{hyperref}
\usepackage[shortlabels]{enumitem}
\usepackage{tikz}
\usetikzlibrary{arrows}
\usepackage{float}

\usepackage[style=numeric,citestyle=numeric,maxbibnames=99,backend=biber,sorting=nyt,natbib=true,backref=true]{biblatex}
\usetikzlibrary{shapes,arrows,positioning}
\tikzstyle{ellipsoid} = [draw, ellipse, minimum height=3em, minimum width=3em]
\tikzstyle{block} = [draw, rectangle, minimum height=3em, minimum width=3em]
\tikzstyle{round} = [draw, circle, minimum height=3em, minimum width=2em]
\tikzstyle{virtual} = [coordinate]

\newcommand{\ba}{\boldsymbol{a}}
\newcommand{\A}{\boldsymbol{A}}

\newcommand{\bd}{\boldsymbol{d}}
\newcommand{\x}{\boldsymbol{x}}
\newcommand{\y}{\boldsymbol{y}}
\newcommand{\bv}{\boldsymbol{v}}
\newcommand{\bL}{\boldsymbol{L}}
\newcommand{\bR}{\boldsymbol{R}}
\newcommand{\w}{\boldsymbol{w}}

\newcommand{\M}{\boldsymbol{M}}

\newcommand{\I}{\boldsymbol{I}}

\newcommand{\X}{\boldsymbol{X}}

\newcommand{\z}{\boldsymbol{z}}

\newcommand{\Q}{\boldsymbol{Q}}

\newcommand{\bu}{\boldsymbol{u}}
\newcommand{\be}{\boldsymbol{e}}

\newcommand{\bmu}{\boldsymbol{\mu}}
\newcommand{\zero}{\boldsymbol{0}}

\newcommand{\bDelta}{\boldsymbol{\Delta}}

\newcommand{\hw}{\hat{\w}}

\newcommand{\hc}{\hat{C}}
\newcommand{\dc}{\check{C}}
\newcommand{\barmu}{\bar{\bmu}}

\newcommand{\nd}{n_{\Delta}}
\newcommand{\R}{\mathbb{R}}

\newcommand{\Risk}{\mathcal{R}}

\newcommand{\sumn}{\sum_{i=1}^{n}}

\newcommand{\umin}[1]{\underset{#1}{\min}}

\newcommand{\umax}[1]{\underset{#1}{\max}}

\newcommand{\uargmin}[1]{\underset{#1}{\arg\min\,}}

\newcommand{\pts}[1]{\left(#1\right)}
\newcommand{\mts}[1]{\left[#1\right]}
\newcommand{\Riskf}[2][]{\Risk_{#1}\left(#2\right)}

\newcommand{\ellf}[2][]{\ell^{#1}\left(#2\right)}
\newcommand{\ellsf}[2][]{\ell_{#1}\left(#2\right)}

\newcommand{\ip}[2]{\left\langle #1, #2\right\rangle}

\newcommand{\Qf}[2][]{Q^{#1}\left(#2\right)}

\newcommand{\ffs}[2][]{f_{#1}\left(#2\right)}

\newcommand{\nnorm}[2][]{\left\|#2\right\|_{#1}}
\newcommand{\normo}[1]{\left|#1\right|}

\newcommand{\seq}[1]{\left\{#1\right\}}
\newcommand{\seqr}[3]{\left\{#1\right\}_{#2}^{#3}}
\newcommand{\expf}[2][]{\exp^{#1}\left(#2\right)}
\newcommand{\logf}[2][]{\log^{#1}\left(#2\right)}

\newcommand{\lbf}[2][]{\lambda_{#1}\left(#2\right)}
\newcommand{\XX}{\X\X^\top}

\newcommand{\XXtaui}{\left(\X\X^\top+\tau\I\right)^{-1}}

\newcommand{\XXz}{\X_0\X_0^\top}

\newcommand{\XXztaui}{\left(\X_0\X_0^\top+\tau\I\right)^{-1}}

\newcommand{\XXt}{\X_2\X_2^\top}

\newcommand{\XXttaui}{\left(\X_2\X_2^\top+\tau\I\right)^{-1}}
\newcommand{\XXk}{\X_k\X_k^\top}

\newcommand{\XXki}{\left(\X_k\X_k^\top\right)^{-1}}
\newcommand{\XXktaui}{\left(\X_k\X_k^\top+\tau\I\right)^{-1}}

\newcommand{\XXkpotaui}{\left(\X_{k+1}\X_{k+1}^\top+\tau\I\right)^{-1}}

\newcommand{\XXXdy}{\X^\top\left(\X\X^\top\right)^{-1}\bDelta^{-1}\y}
\newcommand{\XXXtaudy}{\X^\top\left(\X\X^\top+\tau\I\right)^{-1}\bDelta^{-1}\y}

\newcommand{\QQ}{\Q\Q^\top}
\newcommand{\QQi}{\left(\Q\Q^\top\right)^{-1}}
\newcommand{\QQtaui}{\left(\Q\Q^\top+\tau\I\right)^{-1}}
\newcommand{\Mk}{\M_k}
\newcommand{\Mkmo}{\M_{k-1}}

\newcommand{\Mzi}{\M_0^{-1}}
\newcommand{\Mki}{\M_k^{-1}}
\newcommand{\Mkmoi}{\M_{k-1}^{-1}}
\newcommand{\ndpnd}{\frac{n_+}{\Delta_+} + \frac{n_-}{\Delta_-}}
\newcommand{\ndmnd}{\frac{n_+}{\Delta_+} - \frac{n_-}{\Delta_-}}
\newcommand{\nddpndd}{\frac{n_+}{\Delta_+^2} + \frac{n_-}{\Delta_-^2}}

\newcommand{\suuz}{s_{u,u}^{(0)}}
\newcommand{\siik}{s_{i,i}^{(k)}}
\newcommand{\sijk}{s_{i,j}^{(k)}}
\newcommand{\suuk}{s_{u,u}^{(k)}}
\newcommand{\suik}{s_{u,i}^{(k)}}

\newcommand{\sotk}{s_{1,2}^{(k)}}
\newcommand{\sotkpo}{s_{1,2}^{(k+1)}}
\newcommand{\skkkmo}{s_{k,k}^{(k-1)}}
\newcommand{\siiz}{s_{i,i}^{(0)}}

\newcommand{\sotz}{s_{1,2}^{(0)}}

\newcommand{\siikpo}{s_{i,i}^{(k+1)}}

\newcommand{\sikpok}{s_{i,k+1}^{(k)}}
\newcommand{\sjkpok}{s_{j,k+1}^{(k)}}
\newcommand{\sokpok}{s_{1,k+1}^{(k)}}
\newcommand{\stkpok}{s_{2,k+1}^{(k)}}
\newcommand{\skpokpok}{s_{k+1,k+1}^{(k)}}

\newcommand{\suukpo}{s_{u,u}^{(k+1)}}

\newcommand{\sukpok}{s_{u,k+1}^{(k)}}

\newcommand{\suiz}{s_{u,i}^{(0)}}

\newcommand{\suikpo}{s_{u,i}^{(k+1)}}

\newcommand{\tiik}{t_{i,i}^{(k)}}
\newcommand{\tijk}{t_{i,j}^{(k)}}
\newcommand{\totk}{t_{1,2}^{(k)}}
\newcommand{\tkkkmo}{t_{k,k}^{(k-1)}}
\newcommand{\tiikpo}{t_{i,i}^{(k+1)}}
\newcommand{\tikpok}{t_{i,k+1}^{(k)}}
\newcommand{\tokpok}{t_{1,k+1}^{(k)}}
\newcommand{\ttkpok}{t_{2,k+1}^{(k)}}
\newcommand{\tiiz}{t_{i,i}^{(0)}}
\newcommand{\totz}{t_{1,2}^{(0)}}
\newcommand{\totkpo}{t_{1,2}^{(k+1)}}

\newcommand{\hijk}{h_{i,j}^{(k)}}
\newcommand{\hiuk}{h_{i,u}^{(k)}}

\newcommand{\hijkpo}{h_{i,j}^{(k+1)}}
\newcommand{\hkkkmo}{h_{k,k}^{(k-1)}}
\newcommand{\hikpok}{h_{i,k+1}^{(k)}}
\newcommand{\hokpok}{h_{1,k+1}^{(k)}}
\newcommand{\htkpok}{h_{2,k+1}^{(k)}}
\newcommand{\hijz}{h_{i,j}^{(0)}}

\newcommand{\hkpoik}{h_{k+1,i}^{(k)}}
\newcommand{\hkpojk}{h_{k+1,j}^{(k)}}
\newcommand{\hkpook}{h_{k+1,1}^{(k)}}
\newcommand{\hkpotk}{h_{k+1,2}^{(k)}}

\newcommand{\hiuz}{h_{i,u}^{(0)}}

\newcommand{\hkpouk}{h_{k+1,u}^{(k)}}

\newcommand{\hiukpo}{h_{i,u}^{(k+1)}}

\newcommand{\stdtt}{s_{2_\Delta,2}^{(2)}}

\newcommand{\stdot}{s_{2_\Delta,1}^{(2)}}
\newcommand{\sidjk}{s_{{i_\Delta},j}^{(k)}}

\newcommand{\sidik}{s_{{i_\Delta},i}^{(k)}}
\newcommand{\sidjdk}{s_{{i_\Delta},j_\Delta}^{(k)}}
\newcommand{\skpojdk}{s_{{k+1},j_\Delta}^{(k)}}
\newcommand{\sididk}{s_{i_\Delta,i_\Delta}^{(k)}}
\newcommand{\sidikpo}{s_{i_\Delta,i}^{(k+1)}}
\newcommand{\sididkpo}{s_{i_\Delta,i_\Delta}^{(k+1)}}
\newcommand{\sidkpok}{s_{i_\Delta,k+1}^{(k)}}
\newcommand{\sidiz}{s_{i_\Delta,i}^{(0)}}

\newcommand{\stdok}{s_{2_\Delta,1}^{(k)}}
\newcommand{\stdkpok}{s_{2_\Delta,k+1}^{(k)}}
\newcommand{\stdokpo}{s_{2_\Delta,1}^{(k+1)}}
\newcommand{\stdoz}{s_{2_\Delta,1}^{(0)}}
\newcommand{\sididz}{s_{i_\Delta,i_\Delta}^{(0)}}

\newcommand{\hotdt}{h_{1,2_\Delta}^{(2)}}

\newcommand{\httdt}{h_{2,2_\Delta}^{(2)}}
\newcommand{\hijdk}{h_{i,j_\Delta}^{(k)}}
\newcommand{\hijdkpo}{h_{i,j_\Delta}^{(k+1)}}

\newcommand{\hijdz}{h_{i,j_\Delta}^{(0)}}
\newcommand{\hkpoidk}{h_{k+1,i_\Delta}^{(k)}}
\newcommand{\hkpojdk}{h_{k+1,j_\Delta}^{(k)}}
\newcommand{\hkpotdk}{h_{k+1,2_\Delta}^{(k)}}
\newcommand{\oididk}{o_{i_\Delta,i_\Delta}^{(k)}}

\newcommand{\otdtdt}{o_{2_\Delta,2_\Delta}^{(2)}}
\newcommand{\nus}{\nnorm[2]{\barmu_s}}
\newcommand{\nuc}{\nnorm[2]{\barmu_c}}
\newcommand{\nuo}{\nnorm[2]{\barmu_1}}
\newcommand{\nut}{\nnorm[2]{\barmu_2}}
\newcommand{\nui}{\nnorm[2]{\barmu_i}}

\newcommand{\nuk}{\nnorm[2]{\barmu_k}}
\newcommand{\nukpo}{\nnorm[2]{\barmu_{k+1}}}

\newcommand{\ndo}{\nnorm[2]{\bd_1}}
\newcommand{\ndt}{\nnorm[2]{\bd_2}}
\newcommand{\ndi}{\nnorm[2]{\bd_i}}
\newcommand{\nvi}{\nnorm[2]{\bv_i}}
\newcommand{\nvj}{\nnorm[2]{\bv_j}}
\newcommand{\ndvi}{\nnorm[2]{\bDelta_{-1}\bv_i}}
\newcommand{\ndvj}{\nnorm[2]{\bDelta_{-1}\bv_j}}

\newcommand{\nbu}{\nnorm[2]{\bu}}
\newcommand{\myquad}[1][1]{\hspace*{#1em}\ignorespaces}
\newtheorem{assumption}{Assumption}
\newtheorem{theorem}{Theorem}
\newtheorem{lemma}{Lemma}
\newtheorem{proposition}{Proposition}
\newtheorem{corollary}{Corollary}
\newtheorem{definition}{Definition}
\begin{document}

\begin{center}

{\bf{\LARGE{Sharp analysis of out-of-distribution error for ``importance-weighted" estimators in the overparameterized regime}}}

\vspace*{.2in}

{\large{
\begin{tabular}{ccc}
Kuo-Wei Lai$^\dagger$ & Vidya Muthukumar$^{\dagger,\ddagger}$
\end{tabular}}}

\vspace*{.2in}

\begin{tabular}{c}
School of Electrical \& Computer Engineering, Georgia Institute of Technology$^\dagger$\\
H. Milton School of Industrial \& Systems Engineering, Georgia Institute of Technology$^\ddagger$
\end{tabular}

\vspace*{.2in}
\date{}

\end{center}

\begin{abstract}%
Overparameterized models that achieve zero training error are observed to generalize well on average, but degrade in performance when faced with data that is under-represented in the training sample. In this work, we study an overparameterized Gaussian mixture model imbued with a \emph{spurious feature}, and sharply analyze the \emph{in-distribution} and \emph{out-of-distribution} test error of a \emph{cost-sensitive} interpolating solution that incorporates "importance weights". Compared to recent work~\cite{behnia2022avoid,wang2021importance}, our analysis is sharp with matching upper and lower bounds, and significantly weakens required assumptions on data dimensionality. Our error characterizations also apply to any choice of importance weights and unveil a novel tradeoff between worst-case robustness to distribution shift and average accuracy as a function of the importance weight magnitude\footnote{A short version of this work will be presented at IEEE ISIT 2024.}.
\end{abstract}


\section{Introduction}\label{sec:intro}
Overparameterized models are ubiquitous in machine learning theory and practice today because of their state-of-the-art generalization guarantees (in the sense of low test error) even while perfectly fitting the training data~\cite{zhang2021understanding,deng2022model}.
However, this ``good generalization" property does not extend to test data that is distributed differently from training data, termed~\emph{out-of-distribution} (OOD) data~\cite{sagawa2019distributionally,sagawa2020investigation,zech2018variable}.
A particularly acute scenario arises when the data is drawn as a mixture from multiple groups (each with a different distribution) and some groups are very under-represented in training data~\cite{buolamwini2018gender}.
Under such models, the \emph{worst-group generalization error} can be significantly degraded with respect to the average generalization error on all groups~\cite{behnia2022avoid,wang2021importance,sagawa2020investigation,sagawa2019distributionally}.
%

The effect of distribution shift on generalization has been sharply characterized in a worst-case/minimax sense, e.g.~\cite{mansour2008domain,kpotufe2021marginal,pathak2022new}. However, distribution-shift arising from practical inequities in group representation is often much more structured.
Of particular interest is distribution shift induced by \emph{spurious correlations} present in the
training dataset, which is a scenario that commonly arises on real-world image and language datasets\footnote{For example, the Waterbirds
dataset~\cite{welinder2010caltech} overwhelmingly features landbirds on land and waterbirds on water; in this dataset, the background of the image is a spurious feature.}. Here,
the shifted distribution removes or reverses these correlations, degrading performance.
Recent work~\cite{wang2021importance,behnia2022avoid,sagawa2019distributionally} has examined the ramifications of training overparameterized models in the presence of spurious correlations and proposed several strategies to mitigate their harmful effects on worst-group error.
For example,~\cite{wang2021importance,behnia2022avoid,kini2021label,lai2023general} introduce novel loss functions with \emph{adjustment weights} that upweight the importance of minority group examples. On the other hand, \emph{group distributionally robust optimization (DRO)} techniques~\cite{sagawa2019distributionally,hu2018does,kuhn2019wasserstein} adaptively optimize examples from the worst group during training. 
Despite these efforts, a tight characterization of the impact of overparameterization, number of examples and training loss function on OOD generalization remains missing, even in this special setting.
%

\paragraph{Our contributions:}

In this paper, we provide sharp matching upper and lower bounds on the group-wise generalization error of the overparameterized linear model on data drawn from the Gaussian Mixture Model (GMM) data distribution imbued with a spurious feature: this model was first introduced in~\cite{sagawa2020investigation} and studied in~\cite{behnia2022avoid}.
We study, in particular, a form of \emph{cost-sensitive minimum-norm interpolation} (cMNI) wherein minority examples are ``upweighted" by an adjustment weight $1/\Delta_-$ (our bounds also carry over to the cost-sensitive support-vector-machine (SVM) via a corollary of~\cite{behnia2022avoid}).
Our results (Theorem~\ref{thm:wst_risk_upperbound} and Proposition~\ref{pro:wst_lower_bound}) characterize the precise role of the number of majority and minority training examples, the model dimension, and the upweighting factor on both ID and OOD generalization.
Our tight bounds uncover an interesting \emph{robustness-accuracy tradeoff} (Table~\ref{tab:tradeoff_acc_robust}): while a smaller upweighting factor (larger $\Delta_-$) results in higher average accuracy and worse robustness (i.e. higher worst-group error), a larger upweighting factor (larger $\Delta_+$) results in improved robustness at the cost of average accuracy.
We also demonstrate that explicit regularization in the form of a ridge penalty does not change our bounds except for universal constant factors, nor does it change the nature of this robustness-accuracy tradeoff.

At a technical level, to achieve matching upper and lower bounds we deviate significantly from the related analyses~\cite{behnia2022avoid,wang2021importance}, which study the regularization path of gradient descent on the training error objective directly, and require strong assumptions on the dimensionality of the data.
We instead leverage techniques from benign overfitting in binary and multiclass classification~\cite{wang2022binary,wang2021benign} and apply them to a careful and delicate analysis of the cost-sensitive estimators.

\paragraph{Organization of paper:} The rest of this paper is organized as follows. Section~\ref{sec:relatedwork} contextualizes our approach and results with the most closely related work.
Section~\ref{sec:wst_risk_upperbound} presents our main results, Section~\ref{sec:upperbound_proof} is overall proof, and Section~\ref{sec:conclusion} provides a brief conclusion.
The proofs of technical and auxiliary lemmas are deferred to the appendices.
\subsection{Related work}\label{sec:relatedwork}
The problem of generalizing to \emph{out-of-distribution} (OOD) data has been studied in a vast body of work through two lenses.
First, the goal of \emph{domain adaptation}, or generalizing to an unseen domain, has been studied from a minimax perspective in statistics, e.g.~\citep{kpotufe2021marginal,pathak2022new} and learning theory, e.g.~\citep{mansour2008domain}.
Such minimax bounds are often pessimistic and do not reflect real-world distribution shifts that may be more structured in nature.
An alternative and increasingly popular model for distribution shift consists of a \emph{mixture model} in which data comes from one of a finite set of pre-defined groups, each of which possesses a very different data distribution.
Typically, some groups (designated as \emph{minority groups}) are under-represented in the training data, and a natural type of distribution shift would entail the minority groups being more prevalent, or ``upsampled", in test data.
Under this model, it suffices to minimize the \emph{worst-group error} to achieve robustness to distribution shift.
To model further structure,~\cite{sagawa2020investigation} introduced the Gaussian mixture model with a \emph{spurious feature}.
Here, the data is partitioned into two groups depending on the sign of the spurious feature.
Such spurious features are often present in real-world datasets and induce \emph{spurious correlations} between input and output due to imbalance in the representation of the two groups~\cite{sagawa2020investigation}.
A related setting with \emph{class imbalance} was also studied by~\cite{wang2021importance}. For this setting, researchers have proposed specially designed \emph{group-aware} loss functions to alleviate the impact of OOD data during training, as seen in~\cite{kini2021label,li2021autobalance,wang2021importance,cao2019learning,menon2020long}. Notably,~\cite{kini2021label} introduced the \emph{vector-scaling} loss (VS-loss) and demonstrated both theoretical and empirical benefits in using this loss function for worst-group error~\cite{behnia2022avoid}. The implicit bias of the VS-loss turns out to align with the cost-sensitive minimum-norm solution that we study in this work for high-dimensional data~\cite{behnia2022avoid}.
This makes our findings directly applicable to estimators derived using the VS-loss.
Finally, while our work as well as the above discussion focused on unregularized empirical risk minimization with the gradient descent algorithm, numerous studies utilize adaptive/robust optimization methods~\cite{sagawa2019distributionally,hu2018does,samuel2021distributional}. These methods aim to minimize the worst-case error within predefined groups instead.

Compared to our work, the upper bound on worst-group error of~\cite{behnia2022avoid} is more loose (in particular in its dependence on the number of minority examples); on the other hand, its assumptions on data dimensionality and signal strength are slightly less stringent. 
~\cite{wang2021importance} studied the effect of importance weighting on an interpolating classifier trained with a \emph{polynomial loss function}, and showed that a similar worst-class error scaling is possible to our Theorem~\ref{thm:wst_risk_upperbound} for a specific case of the importance weights.
However, the assumptions on data dimensionality made in~\cite{wang2021importance} are significantly more stringent.
Moreover, neither of these papers provides a matching lower bound on the worst-group risk or theoretically characterizes the impact of importance weight magnitude on ID and OOD error, nor do these papers characterize the impact of explicit ridge regularization.
Simulations conducted in~\cite{wang2021importance} do highlight a robustness-accuracy tradeoff as a function of the magnitude of the importance weight that we uncover mathematically in our work.
See Table~\ref{tab:comparison} for a brief contextualization of our results with these closely related papers.

Our approach is quite different from that taken in in~\cite{wang2021importance,behnia2022avoid}, which both apply the proof strategy proposed in~\cite{chatterji2021finite} which derives bounds between gradient updates of data examples in gradient descent.
While that approach is generally applicable to a variety of training loss functions and label noise models, it requires larger data dimension requirements that can be removed by other analyses even for ID model generalization~\cite{wang2022binary,muthukumar2021classification}. 
We instead take the approach of \emph{benign overfitting in classification of binary/multiclass GMM}~\cite{wang2022binary,wang2021benign,cao2021risk}, which is known to provide sufficient and necessary conditions for ID generalization (i.e. in the absence of spurious features).
The introduction of spurious features and unequal weights on training examples introduces new technical challenges in the analysis even for binary classification. We discuss in detail how we tackle these challenges in Appendix~\ref{app:primitive_proof}.

Finally, we wish to mention that our in-depth contextualization above has focused on the impact of distribution shift in the spurious feature setup with distinct groups, and on classification error.
A parallel question of the impact of covariate shift on regression error in the overparameterized regime has also received sizable recent interest; see, e.g.~\cite{tripuraneni2021overparameterization,tripuraneni2021covariate,feng2024towards}.
This work predominantly studies random-feature models and the impact of overparameterization on robustness.

\vspace{-2mm}
\subsection{Notation}
We use lower-case boldface (e.g. $\x$) to denote vector notation and upper-case boldface (e.g. $\X$) to denote matrix notation.
We use $\nnorm[2]{\cdot}$ to denote the $\ell_2$-norm of a vector and the operator norm of a matrix. We use $\lbf[k]{\X}$ to denote the $k$-th largest eigenvalue of matrix $\X$ for $\X\in\R^{n\times n}$ and $1\leq k \leq n$. We use the shorthand notation $[n]$ to denote the set of natural numbers $\{1,\ldots,n\}$.
We use $C,C_1,C_2,\ldots > 0$ to denote universal and finite constants independent of all problem parameters, that can change line to line.
\begin{table}[htbp]
    \renewcommand{\arraystretch}{1.2}
    \caption{Comparison of our results with related work}
    \label{tab:comparison}
    \centering
    \resizebox{0.85\textwidth}{!}{\begin{tabular}{|c | c | c | c|} 
    \hline
     & $\Riskf[wst]{\hw}$ & Assumed lower bound on $d$ & Assumed lower bound on $R_+$ \\
    \hline\hline
    Wang et al.~\cite{wang2021importance} & $\mathcal{O}(R_+^2n_-/d)$ & $\Omega(n^3\log(n/\delta))$ & $\Omega(n^2\log(n/\delta))$ \\ 
    \hline
    Behnia et al.~\cite{behnia2022avoid} & $\mathcal{O}(R_+^2/d)$ &  $\Omega(n^2\log(n/\delta))$ & $\Omega(\log(n/\delta))$ \\
    \hline
    \textbf{Our result ($\Delta_b = \frac{n_b}{n}$)} & \textbf{$\Theta(R_+^2n_-/d)$} &  \textbf{$\Omega(n^2\log(n/\delta))$} & \textbf{$\Omega(n\log(n/\delta))$} \\ 
    \hline
    \end{tabular}}
\end{table}

\section{Our group-wise error characterizations}\label{sec:wst_risk_upperbound}
%
%
%
\paragraph{Data distribution.}
We use the setup in~\cite{behnia2022avoid}, i.e.~a binary classification problem with data possessing spurious features which incorrectly correlate with its label. We denote $\x\in\R^d$ as a \emph{feature} and $y\in\seq{+1, -1}$ as a \emph{label} which is equal to $1$ with probability $\pi_{+1}$ and $-1$ with probability $\pi_{-1}$. Finally, we denote $a\in\seq{+1, -1}$ as the \emph{attribute} of the data, meaning that the case $a \neq y$ indicates that the example possesses spurious features. Accordingly, $b=y\times a\in\seq{+1, -1}$ partitions the data into two groups distinguishing the presence of spurious features ($b=+1$)  or not ($b=-1$). We consider a product distribution on the features, labels and attributes, i.e.~$\mathcal{P}=\mathcal{P}_{x}\times \mathcal{P}_{y}\times\mathcal{P}_{a}$. We assume that the feature distribution is a \emph{Gaussian mixture model} (GMM) such that $\x$ is a Gaussian vector conditioned on $y$ and $a$ with mean vector $\bmu =\mts{y\bmu_c; a\bmu_s}\in\R^d$ and covariance is $\I_d$, where we denote $\bmu_c$ as the ``correct" class mean vector and $\bmu_s$ is the ``spurious" class mean vector. We also denote $\barmu_c:=\mts{\bmu_c; \zero}\in\R^d$ and $\barmu_s:=\mts{\zero; \bmu_s}\in\R^d$. More formally, we have:
\begin{gather}
    \x|y,a\sim\mathcal{N}\pts{
    \begin{bmatrix}
    y \bmu_c \\
    a \bmu_s 
    \end{bmatrix}, \I_d}, \\
    \x = \begin{bmatrix} \label{eq:x_definition}
    y \bmu_c \\
    a \bmu_s 
    \end{bmatrix} + \z = y\barmu_c+a\barmu_s + \z,
\end{gather}
where $\z\sim\mathcal{N}\pts{\zero, \I_d}$. We also denote the mean vector for each group as $\bmu_b=y\bmu=\mts{\bmu_c; b\bmu_s}$. 
Finally, we denote the \emph{total signal strength} as $R_+=\nnorm[2]{\bmu}^2=\nuc^2 + \nus^2$ and the \emph{signal strength difference} as $R_-=\nuc^2 - \nus^2$.

\paragraph{Training dataset.}
We consider a dataset $\seqr{\x_i, y_i, a_i}{i=1}{n}$ where each example is drawn i.i.d. from the distribution $\mathcal{P}$. We consider the number of \emph{minority examples} (i.e. where $b_i = -1$) to be $n_-$ and the number of \emph{majority examples} (i.e. where $b_i = +1$) to be is $n_+ = n - n_-$. We also denote $\X\in\R^{n\times d}$ as the training data matrix (where each row in $\X$ is a data point), $\y\in\R^n$ as the label vector, $\ba\in\R^n$ as the attribute vector. The definition of $\x$ in Eq.~\eqref{eq:x_definition} gives us
\begin{align}\label{eq:x_matrix_def}
    \X=\y\barmu_c^\top+\ba\barmu_s^\top+\Q,
\end{align}
where $\Q\in\R^{n\times d}$ and each row in $\Q$ follows $\mathcal{N}\pts{\zero, \I_d}$.
%

\paragraph{Assumptions on data.}
We first make some mild assumptions: a) we operate in the \emph{overparameterized regime} ($d \geq n$) in which the dataset is linearly separable almost surely; b) the number of majority examples exceeds the minority examples ($n_+ \geq \frac{n}{2}$), and c) the strength of the correct features exceeds the strength of the spurious features ($R_- \geq 0$).
Next, we introduce our stronger assumptions on the dataset that are required to sharply analyze the group-wise classification error. These assumptions are considerably weaker than those made in~\cite{wang2021importance} and mostly identical to those made in~\cite{behnia2022avoid}, with the exception of Assumption~\ref{asm:dataset_gen}(B) which is slightly stronger.
%
\begin{assumption}\label{asm:dataset_gen}
    For a target failure probability $\delta$, there exists a large enough constant $C>0$ such that
    \begin{multicols}{2}
    \renewcommand{\labelenumi}{\Alph{enumi})}
    \begin{enumerate}
        \item $n \geq C \logf{1/\delta}$,
        \item $\nnorm[2]{\barmu_c}^2 \geq Cn\logf{n/\delta}$,
        \item $d \geq CR_+n$,
        \item $d \geq Cn^2\logf{n/\delta}$.
    \end{enumerate}
    \end{multicols}
\end{assumption}
\paragraph{Estimators.}
We deploy a linear model $\w\in\R^d$ that solves \emph{empirical risk minimization (ERM) with ridge regularization}, i.e.~$\umin{\w\in\R^d}\Riskf{\w} = \frac{1}{n}\sumn\ellf{y_i, \ip{\w}{\x_i}} + \tau\nnorm[2]{\w}^2,$ where $\ell$ is the loss function and $\tau\geq0$. Note that when $\tau=0$, we recover the special case of unregularized ERM, which will result in an interpolating solution. We begin our analysis with the squared loss with \emph{adjusted labels} for this classification problem, $\ellsf[\text{sq}]{y, z} = \pts{\Delta^{-1}y - z}^2$, where $0 < \Delta \leq 1$ is the \emph{adjustment weight} of the label (also commonly called \emph{importance weight}). Typically, we adjust the weights $\Delta$ differently for majority and minority examples to mitigate spurious correlations. Accordingly, we denote $\Delta_+$ and $\Delta_-$ as the adjustment weight for majority examples and minority examples. This results in the closed-form \emph{cost-sensitive ridge estimator}
\begin{gather}\label{eq:sq_ridge_sol}
    \hw_{\text{ridge}}:=\XXXtaudy,
\end{gather}
where $\bDelta\in\R^{n\times n}$ is the diagonal matrix of adjustment weights, i.e.~$\bDelta_{i,i} = \Delta_{b_i}$.
%
Since $d \geq n$ and we assume Gaussian covariates, $\X$ is almost surely full-rank.
Therefore, in the special case where $\tau = 0$, solving the unregularized ERM with gradient descent (initialized at $0$) would yield the \emph{cost-sensitive minimum-norm interpolation} (cMNI) of the adjusted labels, i.e.
%
\begin{gather}\label{eq:mni_sol}
    \hw_{\text{cMNI}}:=\uargmin{\w}\nnorm[2]{\w}\\
    \text{ s.t. } \Delta_{b_i}\ip{\w}{\x_i} = y_i \text{ for all } i\in[n]\label{eq:mni_constraint},
\end{gather}
This solution turns out to be expressible in closed form as:
\begin{gather}\label{eq:sq_sol}
    \hw_{\text{cMNI}}:=\XXXdy.
\end{gather}
%
Note that decreasing the adjustment weight $\Delta_{b_i}$ would induce a larger \emph{margin} on the example $\x_i$ in Eq.~\eqref{eq:mni_constraint}, which could alleviate at least in part the harmful effect of spurious correlations.
On the other hand, very low adjustment weights could also increase the variance of the cMNI estimator.
As a final remark, we note that under Assumption~\ref{asm:dataset_gen},~\cite{behnia2022avoid} shows that the cost-sensitive support-vector-machine (cSVM) solution (arising from gradient descent run on a logit-adjusted loss~\cite{kini2021label}) exactly coincides with the cMNI solution presented in Eq.~\eqref{eq:mni_sol}. Therefore, our forthcoming characterization of the cMNI estimator also automatically applies to the cSVM estimator, which is more popular in practice~\cite{kini2021label}. 

%
\paragraph{Generalization error.}
For a fresh (test) sample $\pts{\x, y, a}$, we can calculate the expected 0-1 error risk as $\Riskf{\hw}=\mathds{E}\mts{\mathds{1}\pts{y\neq\hat{y}}}=\text{Pr}\mts{y\hw^\top \x < 0}.$
%
%
Further, we consider the worst-group error, defined below as
\begin{align}\label{eq:worst_group_risk}
    \Riskf[wst]{\hw}=\umax{b\in\pm1}~\Riskf[b]{\hw}=\umax{b\in\pm1}~\text{Pr}\mts{y\hw^\top \x < 0 | y \times a = b}.
\end{align}
Our main focus in this work is deriving tight upper and lower bounds for Eq.~\eqref{eq:worst_group_risk}.

\subsection{Error analysis results}
%
We present our main result, an upper bound for the group-wise generalization error, in Theorem~\ref{thm:wst_risk_upperbound} below.
\begin{theorem}\label{thm:wst_risk_upperbound}
    Under Assumption~\ref{asm:dataset_gen}, the generalization error for each group $b \in \{+1,-1\}$ is upper bounded as 
    \begin{gather}
        \Riskf[b]{\hw}\leq\expf{-C_1\frac{ \pts{\alpha_+ R_b^2n_+ + \alpha_-R_{-b}^2n_-}}{d}},
    \end{gather}
    with probability at least $1 -\delta$, where we defined $\alpha_\pm \coloneqq \frac{n_\pm/\Delta_\pm^2}{\nddpndd}=\frac{n_\pm/\Delta_\pm^2}{\nd}$ and we have $0 < \alpha_\pm < 1$.
\end{theorem}
Notice that Theorem~\ref{thm:wst_risk_upperbound} characterizes the group-wise error as a function of: a) the number of majority and minority examples $n_+,n_-$, b) the total and difference of signal strength $R_+, R_-$, c) the number of features $d$, and d) the parameter $\alpha_\pm$ which is a functional of the adjustment weights $\Delta_\pm$ (and $n_+,n_-$). 
The forthcoming Proposition~\ref{pro:wst_lower_bound} shows that this rate is sharp up to a universal constant factor in all of these problem parameters.
\begin{proposition}\label{pro:wst_lower_bound}
Under Assumption~\ref{asm:dataset_gen}, the generalization error for each group $b \in \{+1,-1\}$ is lower bounded as
    \begin{gather}
        \Riskf[b]{\hw}\geq C_2 \expf{- C_3 \frac{ \pts{\alpha_+R_b^2n_+ + \alpha_-R_{-b}^2n_-}}{d}},
    \end{gather}
    where $\alpha_\pm \coloneqq \frac{n_\pm/\Delta_\pm^2}{\nddpndd}=\frac{n_\pm/\Delta_\pm^2}{\nd}$ and $0 < \alpha_\pm < 1$.
\end{proposition}
It is instructive to consider a few special cases of Theorem~\ref{thm:wst_risk_upperbound} and Proposition~\ref{pro:wst_lower_bound}.
First, if there is no spurious feature in the data (i.e. $\barmu_s = \mathbf{0}$), it is easy to verify that we recover the \emph{in-distribution} error rate for overparameterized binary GMM in~\cite{wang2022binary}.
Second, if the magnitude of the correct and spurious means are near-identical, i.e.~$\nuc \approx \nus$ or $R_- \approx 0$, we obtain the rates $\Riskf[b]{\hw} \asymp \expf{-\frac{C \alpha_b R_+^2 n_b }{d}}$. Therefore, decreasing $\Delta_b$, increasing total signal strength $R_+$, or increasing $n_b$ would improve the error rate. 
Second, we note that the ridge regularization parameter $\tau$ only affects the bound through its influence on the constants $C_1$ and $C_3$, and $0 < C_1,C_3 < \infty$ for all values of $\tau$.
Therefore, the non-asymptotic dependence of the bound on the parameters $d,n_b,R_b,\alpha_b$ is, interestingly, unaffected by explicit ridge regularization.
All the same, Figure~\ref{fig:error_improve} shows that the empirical worst-group error decreases with an increase in $\tau$. This phenomenon is not captured by our proof technique, as we do not provide exact asymptotic characterizations of the test error\footnote{We do expect that our non-asymptotic bounds would be tighter for higher values of $\tau$, meaning that $C_1$ would \emph{increase} with $\tau$ and $C_3$ would \emph{decrease} with $\tau$ --- but we do not conduct a detailed formal investigation of the exact impact of $\tau$ on the constants in our bounds in this paper.}.
Corollary~\ref{cor:vanish_bound} below discusses sufficient and necessary conditions under which we would obtain vanishing error, i.e.~\emph{consistency} for out-of-distribution generalization.
\begin{figure}[htbp]
\centering 
 \includegraphics[width=80mm]{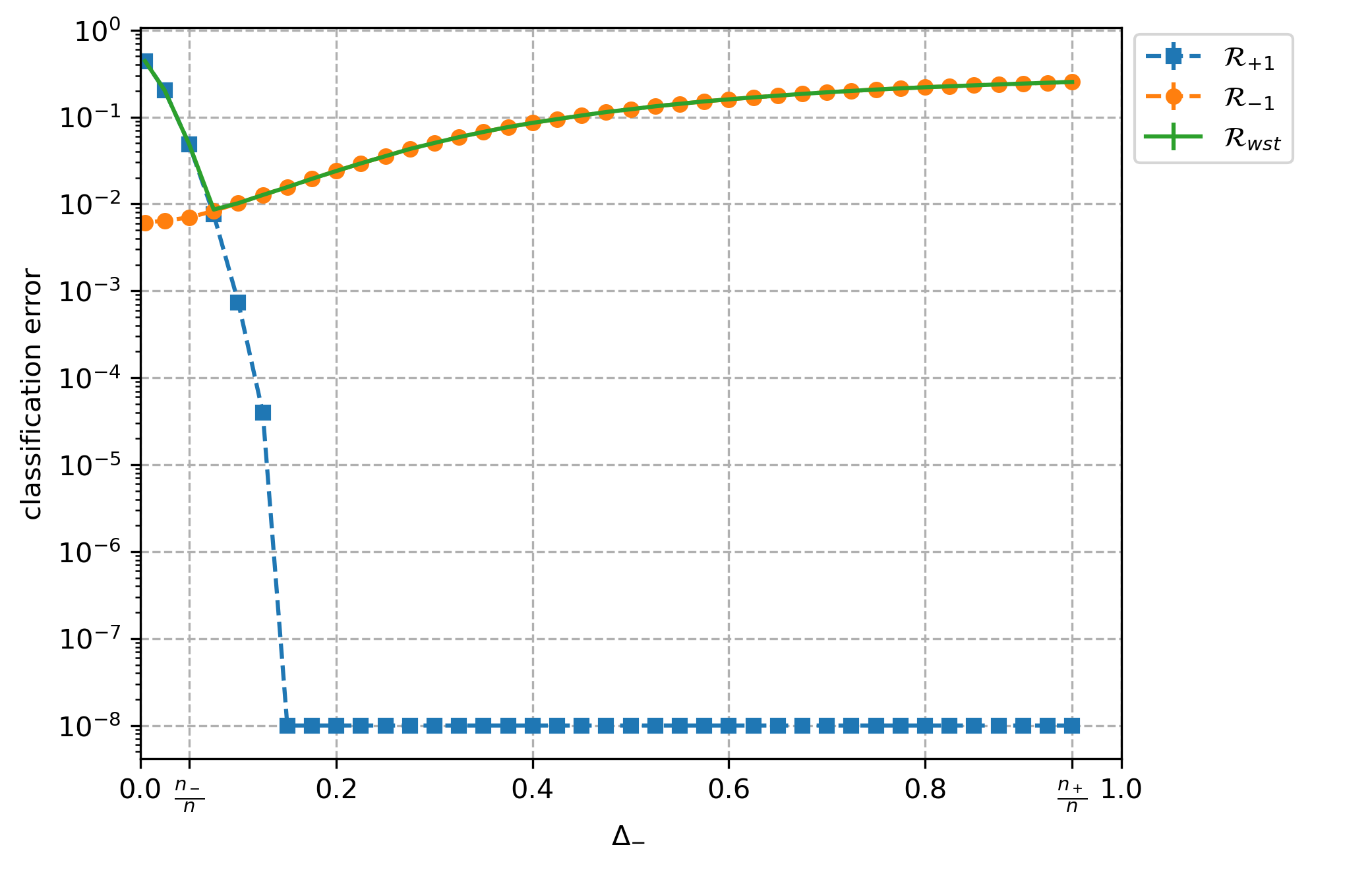}
 \includegraphics[width=68mm]{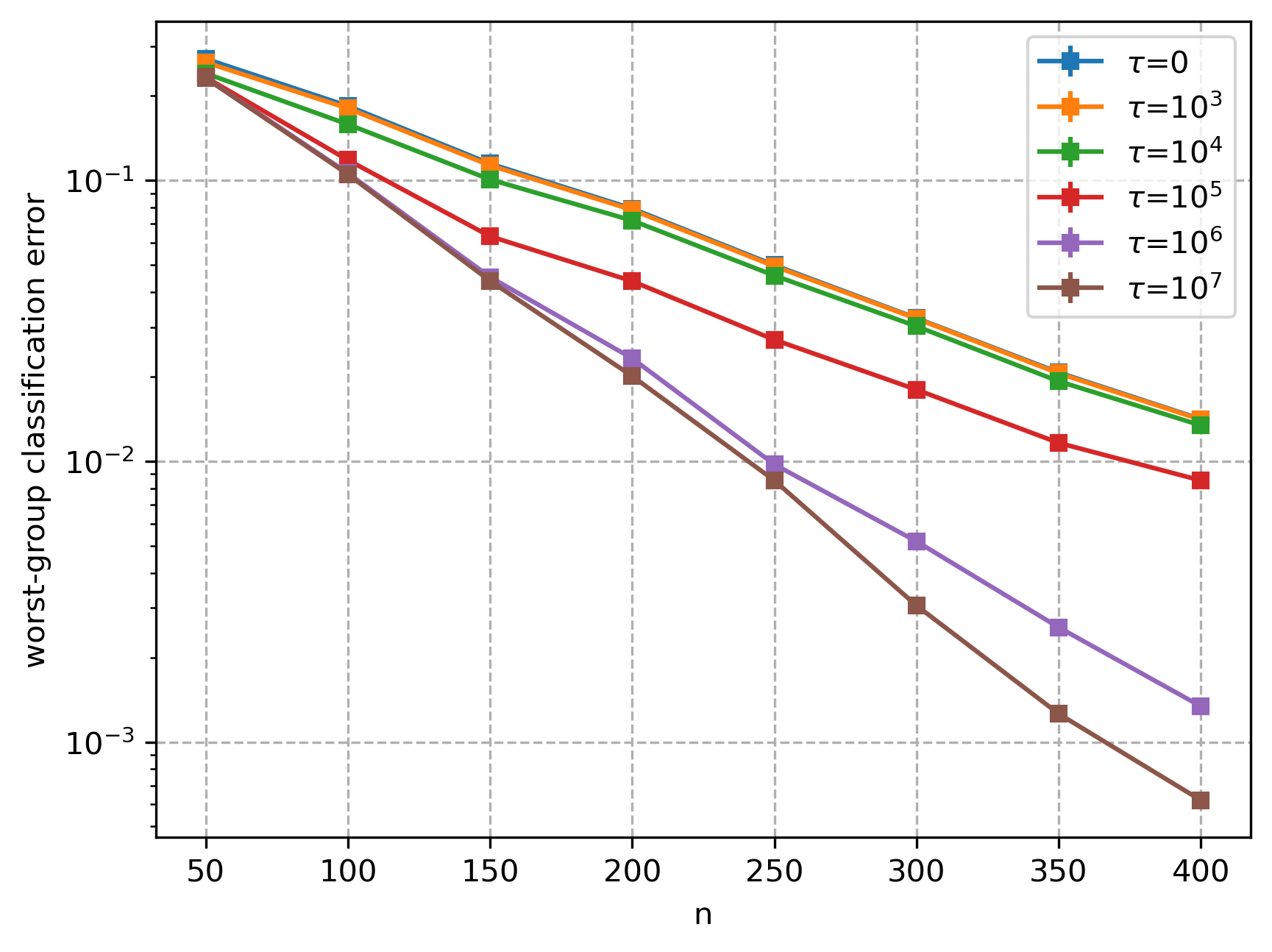}
 \caption{The left panel plots group-wise error as a function of $\Delta_{-}$.
 We fix $d=10^5$, $n=200$, $n_-=10$, $R_+=d^{0.6}/4$, $\bmu_c=\bmu_s=\sqrt{R_+/2}\be_1$ (as in~\cite{behnia2022avoid}), $\Delta_+=\frac{n_+}{n}$ and make $\Delta_-$ decrease from $\frac{n_+}{n}$ to $\frac{n_-}{n}.$ Observe that the worst-group error decreases when $\Delta_-$ decreases until $\Delta_-=\frac{n_-}{n}$, below which the worst-group becomes the majority group, whose error \emph{increases} with decreased $\Delta_{-}$. The right panel plots the worst-group error as a function of $n$ fixing $d=2n^2$, $n_-=0.04n$, $R_+=d^{0.6}/4$, $\bmu_c=\bmu_s=\sqrt{R_+/2}\be_1$, $\Delta_\pm=n_\pm/n$. Observe that increasing the ridge regularization parameter $\tau$ improves the worst-group error rate, but only up to a constant factor in the error exponent. These simulations were obtained by averaging over $10$ trials, and error bars are small enough to not be visible.} \label{fig:error_improve}
\end{figure}

\begin{corollary}\label{cor:vanish_bound}
    Consider the case where $R_- = 0$. Then, under the conditions of Assumption~\ref{asm:dataset_gen}, the generalization error for group $b$ vanishes to $0$ as $\frac{d}{n_b}\rightarrow\infty$ if and only if $R_+^2 = \Omega\left(\frac{d}{\alpha_b n_b}\right)$.
\end{corollary}
Intuitively, the condition on $R_+^2 = \Omega\left(\frac{d}{\alpha_b n_b}\right)$ ensures that the total signal strength is sufficiently large, which reduces test error. 
Corollary~\ref{cor:vanish_bound} further shows how the adjustments $\Delta_\pm$, through their influence on the fractions $\alpha_b$ (and therefore the adjustments $\Delta_\pm$) can affect the condition on total signal strength for achieving vanishing group-wise test error.
The implications of Corollary~\ref{cor:vanish_bound} are summarized in Table~\ref{tab:tradeoff_acc_robust} for two extreme choices of adjustment weights: a) $\Delta_\pm = 1$ (i.e. no adjustments) and b) $\Delta_\pm = \frac{n_{\pm}}{n}$ (i.e. importance weighting). 
If we choose $\Delta_\pm=1$, we get $\alpha_b = n_b/n$ and we need $R_+^2 = \Omega\left(\frac{dn}{n_b^2}\right)$ to have vanishing error. Then, since $n_+ \geq \frac{n}{2}$, the majority group generalization error vanishes \emph{iff} $R_+^2 = \Omega(\frac{d}{n})$, which matches the optimal rate of the case without spurious features~\cite{wang2022binary}. However, the condition for vanishing minority-group error becomes much more stringent, and in the worst case yields $R_+^2 = \Omega(dn)$ when $n_-$ is a constant. Thus, $\Delta_\pm = 1$ yields optimal majority-group error but highly suboptimal worst-group error.
On the other hand, if we use $\Delta_\pm = n_\pm/n$, we get $\alpha_b \asymp n_{-b}/n$ and can show that both the majority-group error and worst-group error would vanish \emph{iff} $R_+^2 = \Omega\left(\frac{d}{n_-}\right)$. As a result, this choice improves worst-group error but worsens majority-group error.
Thus Table~\ref{tab:tradeoff_acc_robust} highlights an interesting tradeoff between \emph{average accuracy} and \emph{worst-group robustness} as a function of the adjustment weight magnitudes. 
We briefly explore this tradeoff empirically in Figure~\ref{fig:tradeoff},
which shows a sharp contrast in the rates of decrease of majority-group and minority-group error as a function of $R_{+}^2$ for two choices of the importance weights $\Delta_{\pm}$.
This supports the scalings provided in Table~\ref{tab:tradeoff_acc_robust}.
%
%
\begin{table}[htbp]
    \renewcommand{\arraystretch}{1}
    \caption{$\Delta_\pm$ changes the conditions for $\mathcal{R}_b\rightarrow 0$}
    \label{tab:tradeoff_acc_robust}
    \centering
    \resizebox{0.7\columnwidth}{!}{\begin{tabular}{|c | c | c |} 
    \hline
     & $\Delta_\pm=1$ & $\Delta_\pm = n_\pm / n$ \\
    \hline\hline
    $\mathcal{R}_b$ & $ \asymp\expf{-Cn_b^2R_+^2/dn}$ & $\asymp \expf{-Cn_bn_{-b}R_+^2/dn}$\\
    \hline
    $\Riskf[+1]{\hw}\rightarrow0$ & $R_+^2=\Omega(d/n)$ & $R_+^2=\Omega(d/n_-)$ \\ 
    \hline
    $\Riskf[-1]{\hw}\rightarrow0$ & $R_+^2=\Omega(dn/n_-^2)$ &  $R_+^2=\Omega(d/n_-)$\\
    \hline
    \end{tabular}}
\end{table}

\begin{figure}[htbp]
\centering 
 \includegraphics[width=80mm]{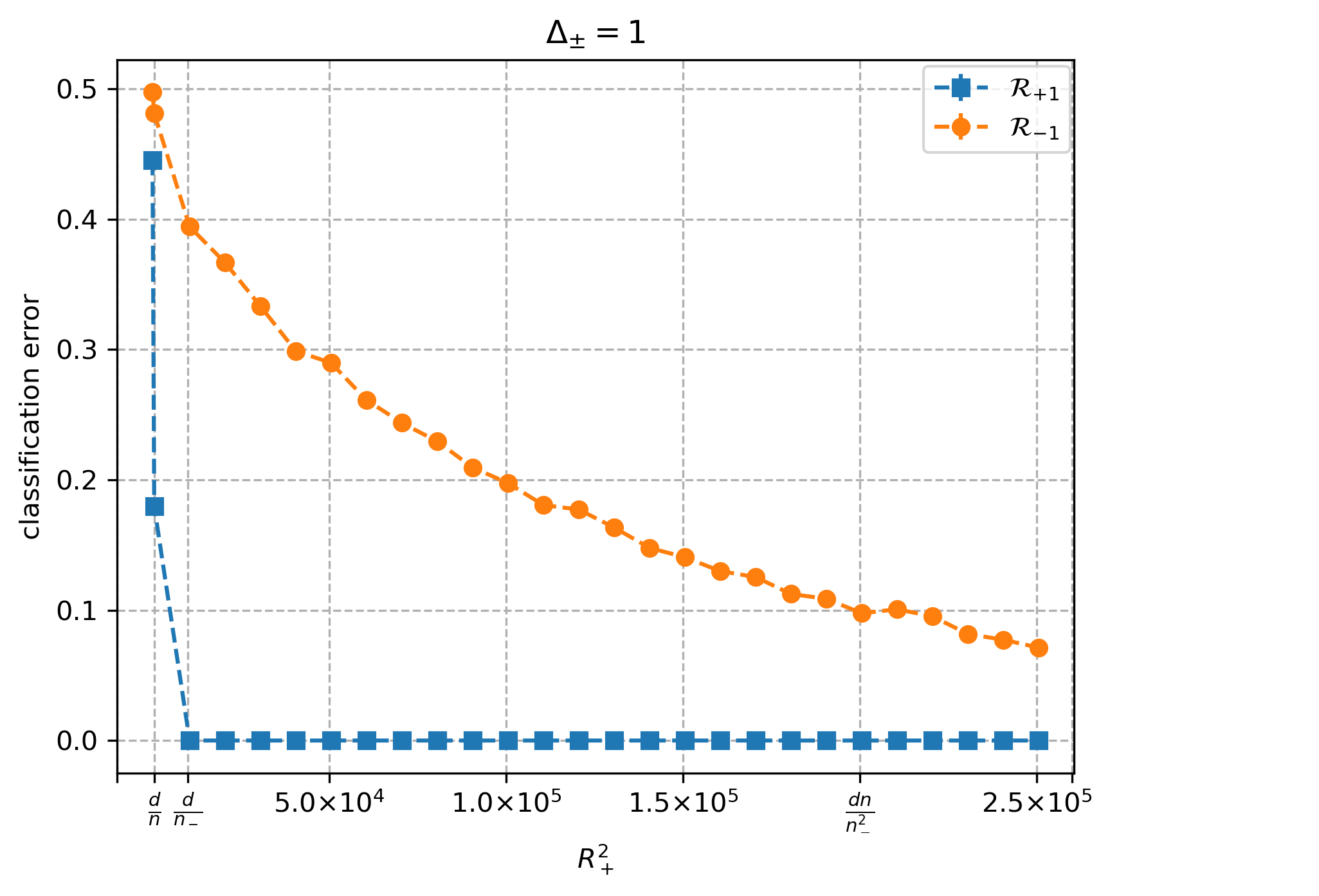}
 \includegraphics[width=80mm]{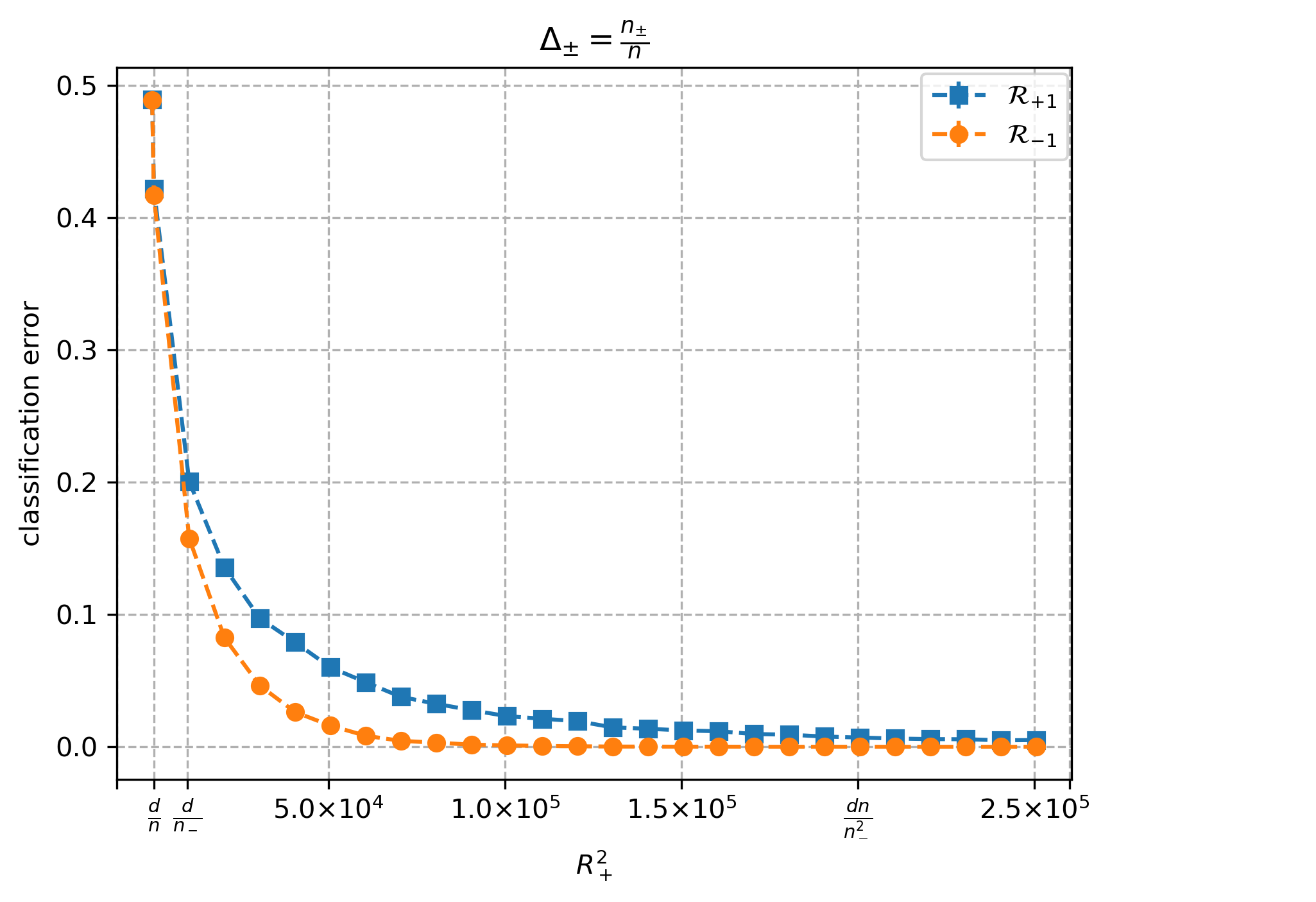}
 \caption{The left panel plots the majority-group and minority-group error as a function of the total signal strength squared ($R_{+}^2$) when the importance weights are set to be equal, i.e.~$\Delta_\pm=1$. We see that while $\Riskf[+1]{\hw}$ approaches $0$ when $R_+^2=\frac{d}{n}$, $\Riskf[-1]{\hw}$ is greater than $0.1$ until $R_+^2=\frac{dn}{n_-^2}$. The right panel also plots group-wise error for the alternative choice $\Delta_\pm=\frac{n_\pm}{n}$. In this case, as expected, both $\Riskf[+1]{\hw}$ and $\Riskf[-1]{\hw}$ decay to $0$ at similar rates. For both plots, we fix $d=10^5$, $n=200$, $n_-=10$, $\bmu_c=\bmu_s=\sqrt{R_+/2\be_1}$. These simulations were obtained by averaging over $10$ trials, and error bars are small enough not to be visible.} \label{fig:tradeoff}
\end{figure}

%

\section{Proof of Theorem~\ref{thm:wst_risk_upperbound}} \label{sec:upperbound_proof}
In this section, we present the proof of Theorem~\ref{thm:wst_risk_upperbound} in four steps. (Note that the lower bound in Proposition~\ref{pro:wst_lower_bound} follows via a matching sequence of steps and uses the matching exponential lower bound of the Q-function~\cite{chiani2003new}; the full proof is in Appendix~\ref{app:lower_bound}.) We first present the upper bound of the worst group generalization error in analytical form. In step 2, we characterize the Gram matrix $\XX$ and its inverse recursively.
Next, to simplify calculations, we define ``primitive" quadratic terms and derive their bounds in step 3. Finally, we substitute these bounds on the primitive terms in the generalization error upper bound to complete the upper bound proof (Theorem~\ref{thm:wst_risk_upperbound}).
We note that steps 2 \& 3 introduce new technical challenges, which we discuss in Appendix~\ref{app:primitive_proof}.

\textbf{Step 1.}
We restate~\cite[Lemma 1]{wang2022binary} below which upper bounds the worst group generalization error.
\begin{lemma}\label{lem:wst_risk_q_upper}
Consider a classification problem under GMM with a linear classifier $\hw$, the worst group risk is $\Riskf[wst]{\hw}=\umax{b\in\pm1}~\Qf{\frac{\hw^{\top}\bmu_b}{\sqrt{\hw^{\top}\hw}}}$. Specifically, if $\hw^{\top}\bmu_b > 0$, then 
\begin{gather}\label{eq:wst_risk_q_upper}
    \Riskf[wst]{\hw} \leq \umax{b\in\pm1}~\expf{-\frac{\pts{\hw^\top\bmu_b}^2}{2\hw^{\top}\hw}}.
\end{gather}
\end{lemma}
According to Eq.~\eqref{eq:wst_risk_q_upper}, it is sufficient to derive a sharp lower bound on $\frac{\pts{\hw^\top\bmu_b}^2}{2\hw^{\top}\hw}$ for each $b\in\pm1$.
Next, we substitute the closed form interpolation of the ridge estimator (Eq.~\eqref{eq:sq_ridge_sol}) into Eq.~\eqref{eq:wst_risk_q_upper} to get 
\begin{gather}\label{eq:target_term_repeat}
    \frac{\pts{\hw^\top\bmu_b}^2}{2\hw^{\top}\hw}
    =\frac{\pts{\y^\top\bDelta^{-1}\XXtaui\X\bmu_b}^2}{2\y^\top\bDelta^{-1}\XXtaui\XX\XXtaui\bDelta^{-1}\y}.
\end{gather}
%
%
\textbf{Step 2.} In order to analyze and lower bound Eq.~\eqref{eq:target_term_repeat}, we first analyze the Gram matrix $\XX$ and its inverse. Recalling the definition of $\X$ in Eq.~\eqref{eq:x_matrix_def}, we define
\begin{align*}
    \M \coloneqq \XX +\tau\I = (\bv_2\barmu_2^\top+\bv_1\barmu_1^\top+\Q)(\bv_2\barmu_2^\top+\bv_1\barmu_1^\top+\Q)^\top +\tau\I,
\end{align*}
where we denote $\bv_1 \coloneqq \ba$, $\bv_2 \coloneqq \y$, $\barmu_1 \coloneqq \barmu_s$, $\barmu_2 \coloneqq \barmu_c$ as shorthand. 
The presence of core and spurious mean vectors (along with the standard Wishart matrix $\QQ$) in the expression for $\XX$ makes analyzing its inverse challenging, even for binary labels.
Inspired by the idea in~\cite{wang2021benign} (which was introduced for \emph{multiclass} classification), we can decompose the data matrix $\X$ as
\begin{gather}
    \X_0 \coloneqq \Q, \quad \X_1 \coloneqq \bv_1\barmu_1^\top + \X_0, \quad \X =: \X_2 =  \bv_2\barmu_2^\top + \X_1\nonumber.
\end{gather}
Then, for $k = 1,2$, we denote $\M_0\coloneqq\XXz+\tau\I = \QQ+\tau\I$ and
\begin{align}
    \Mk&\coloneqq\XXk+\tau\I = \pts{\bv_k\barmu_k^\top+\X_{k-1}}\pts{\bv_k\barmu_k^\top+\X_{k-1}}^\top+\tau\I\nonumber\\
         &= \Mkmo+\underbrace{\mts{\nuk\bv_k,~\bd_k,~\bv_k}}_{=:\bL_k}\underbrace{\begin{bmatrix}
         \nuk\bv_k^\top \\
         \bv_k^\top \\
         \bd_k^\top
         \end{bmatrix}}_{=:\bR_k},\label{eq:xx_def}
\end{align}
%
%
where $\bL_k$, $\bR_k$ are defined in the above display, and we define $\bd_k:=\X_{k-1}\barmu_k=\Q\barmu_k$. Next, as in~\cite{wang2021benign}, we apply the \emph{recursive Woodbury identity}~\cite{horn2012matrix} to break the matrix inversion calculation into recursive steps as below:
\begin{align}
    &\Mzi = \XXztaui = \QQtaui\nonumber\\
    &\Mki = \Mkmoi - \Mkmoi \bL_{k}\A_k^{-1}\bR_k\Mkmoi,\label{eq:xxi_def}
\end{align}
where $\A_k \coloneqq \I+\bR_k\Mkmoi\bL_k$ for $k=1,2$. As a result, we can represent the $k$th order inverse Gram matrix $\XXktaui$ as a constant multiple of the lower-order inverse Gram matrices. Ultimately, this shows that $\XXtaui$ is a constant multiple of the $0$th order \emph{inverse Wishart matrix} $\QQtaui$.\\
\textbf{Step 3.} 
Inspired by~\cite{wang2022binary,wang2021benign}, we define several ``primitive" quadratic terms involving $\XXktaui$ for $k = 0,1,2$. In Step 4 of the proof we will see that Eq.~\eqref{eq:target_term_repeat} can be characterized entirely as a function of these primitives.
\begin{definition}\label{def:primitives}
    For $i, j = 1,2$ and order $k = 0,1,2$, we define the following basic primitives:
    \begin{align}
        &\sijk \coloneqq \bv_i^\top\XXktaui\bv_j,\quad \tijk \coloneqq \bd_i^\top\XXktaui\bd_j, \quad \hijk \coloneqq \bd_i^\top\XXktaui\bv_j \nonumber,\\
        &\suuk \coloneqq \bu^\top\XXktaui\bu,\quad \suik \coloneqq \bu^\top\XXktaui\bv_i,\quad \hiuk \coloneqq \bd_i^\top\XXktaui\bu,\nonumber
    \end{align}
    where we let $\bu\in\R^n$ be an arbitrary unit vector, i.e.~$\nnorm[2]{\bu}=1$.
\end{definition}
%
%
\begin{definition}\label{def:w_primitives}
    For $i,j = 1,2$ and order $k = 0,1,2$, we define the following adjusted primitives:
    \begin{gather}
        \sidjk \coloneqq \bv_i^\top\bDelta^{-1}\XXktaui\bv_j,\nonumber\quad \hijdk\coloneqq \bd_i^\top\XXktaui\bDelta^{-1}\bv_j,\nonumber\\
        \sidjdk \coloneqq \bv_i^\top\bDelta^{-1}\XXktaui\bDelta^{-1}\bv_j,\nonumber\\
        \oididk \coloneqq \bv_i^\top\bDelta^{-1}\XXktaui\XXk\XXktaui\bDelta^{-1}\bv_i.\nonumber
    \end{gather}
\end{definition}
Using the ideas in Step 2, these primitives consisting of quadratic forms on $\XXktaui$ can be characterized sharply as quadratic forms on $\QQtaui$.
Then, we can apply sharp concentration inequalities on the inverse Wishart matrix $\QQi$ to characterize the primitives. 
These bounds are presented in Lemma~\ref{lem:primitive_bounds} and Lemma~\ref{lem:w_primitive_bounds} below; the proofs of these lemmas are contained in Appendix~\ref{app:primitive_proof}.
Henceforth, we define $\hc_m = \frac{C_m+ 1}{C_m}, \dc_m = \frac{C_m - 1}{C_m}$ for universal constants $C_m > 1$.
%
\begin{lemma}\label{lem:primitive_bounds}
    We have the following upper and lower bounds for the basic primitives (Def.~\ref{def:primitives}):
    \begin{gather*}
        \frac{\dc_1n}{d+\tau}\leq \siik \leq  \frac{\hc_1n}{d+\tau}, \myquad[2] \tiik \leq \frac{C_2 n\nui^2}{d+\tau}\\
        \dc_3\frac{n_+}{d+\tau} - \hc_3\frac{n_-}{d+\tau} \leq \sotk \leq \hc_3\frac{n_+}{d+\tau} - \dc_3\frac{n_-}{d+\tau}\\
        -\frac{C_4 n\nuo\nut}{d+\tau}\leq \totk \leq \frac{C_4 n\nuo\nut}{d+\tau}\\
        -\frac{C_5 n\nui}{d+\tau}\leq \hijk \leq \frac{C_5 n\nui}{d+\tau}\\
        \frac{\dc_6}{d+\tau}\leq \suuk \leq  \frac{\hc_6}{d+\tau}, \myquad[1]-\frac{C_7\sqrt{n}}{d+\tau} \leq \suik \leq \frac{C_7\sqrt{n}}{d+\tau}\\
        -\frac{C_8 \sqrt{n}\nui}{d+\tau}\leq \hiuk \leq \frac{C_8 \sqrt{n}\nui}{d+\tau}
    \end{gather*}
\end{lemma}
\begin{lemma}\label{lem:w_primitive_bounds}
    We have the following upper and lower bounds for the adjusted primitives (Def.~\ref{def:w_primitives}):
    \begin{gather*}
        \frac{\pts{\hc_9-\frac{1}{C_{10}}}\pts{\ndpnd}}{d+\tau}\leq \sidik \leq \frac{\pts{\hc_9+\frac{1}{C_{10}}}\pts{\ndpnd}}{d+\tau}\\
        \frac{\hc_9 \pts{\ndmnd} - \frac{1}{C_{10}}\pts{\ndpnd}}{d+\tau}\leq \stdok \leq \frac{\hc_9 \pts{\ndmnd} + \frac{1}{C_{10}}\pts{\ndpnd}}{d+\tau}\\
        \frac{\dc_{11}\nd}{d+\tau}\leq \sididk \leq \frac{\hc_{11} \nd}{d+\tau}\\
        -\frac{C_{12} \sqrt{n\nd}\nui}{d+\tau}\leq \hijdk \leq \frac{C_{12} \sqrt{n\nd}\nui}{d+\tau}\\
        \frac{\dc_{13}\nd d}{\pts{d+\tau}^2}\leq \oididk \leq \frac{\hc_{13} \nd d}{\pts{d+\tau}^2},
    \end{gather*}
    where we define $\nd \coloneqq \frac{n_+}{\Delta_+^2}+\frac{n_-}{\Delta_-^2}$.
\end{lemma}
%
%
%
\textbf{Step 4.} Finally, we can write Eq.~\eqref{eq:target_term_repeat} entirely as a function of the primitives in Step 3. We show the derivation for the case $b= +1$ ($b=-1$ follows by an identical set of calculations):
\begin{align}
    \frac{\pts{\hw^\top\bmu_{+1}}^2}{2\hw^{\top}\hw} &= \frac{\pts{\y^\top\bDelta^{-1}\XXtaui\X\bmu_{+1}}^2}{2\y^\top\bDelta^{-1}\XXtaui\XX\XXtaui\bDelta^{-1}\y}\nonumber\\
    &= \frac{\pts{\bv_2^\top\bDelta^{-1}\XXttaui\pts{\bv_2\barmu_2^\top + \bv_1\barmu_1^\top+\Q}\pts{\barmu_2 + \barmu_1}}^2}{2\bv_2^\top\bDelta^{-1}\XXttaui\XXt\XXttaui\bDelta^{-1}\bv_2}\nonumber\\
    &= \frac{\pts{\nut^2\stdtt + \nuo^2\stdot + \httdt + \hotdt}^2}{2\otdtdt}\label{eq:maj_target_primitive}.
\end{align}
Above, the second equality follows by the definition of $\bmu_{+1}=\barmu_c + \barmu_s$ and the definition of $\X$; the last equality substitutes the adjusted primitives from Def.~\ref{def:w_primitives}. 
We also define the following auxiliary lemma that helps us identify lower-order terms that can be removed up to constant factors. The proof of Lemma~\ref{lem:aux_compare_a_h} is deferred to Appendix~\ref{app:aux_lem_proof}.
\begin{lemma}\label{lem:aux_compare_a_h}
    For some $C>C_1^2>0$, if Assumption~\ref{asm:dataset_gen}(B) holds, or $n_-$ is a factor of $n$ such that $n_-=\alpha n$ for $0 < \alpha \leq 1/2$ is satisfied, we have $\frac{1}{2}\pts{\ndpnd}\nut \geq C_1\sqrt{n\nd}$.
\end{lemma}
We proceed to introduce the primitive bounds from Lemma~\ref{lem:w_primitive_bounds} into Eq.~\eqref{eq:maj_target_primitive} (lower bounding the numerator terms and upper bounding the denominator terms) to get:
%
\begin{align}
    \eqref{eq:maj_target_primitive}&\geq \dfrac{\pts{\nut^2 \frac{\pts{\hc_9-\frac{1}{C_{10}}} \pts{\ndpnd}}{d+\tau} + \nuo^2 \frac{\hc_9 \pts{\ndmnd} - \frac{1}{C_{10}}\pts{\ndpnd}}{d+\tau} -\frac{C_{12} \sqrt{n\nd}\nut}{d+\tau}-\frac{C_{12} \sqrt{n\nd}\nuo}{d+\tau}}^2}{\frac{\hc_{13} \nd d}{\pts{d+\tau}^2}}\nonumber\\
    &= \dfrac{\pts{\pts{\hc_9\nut^2 -\frac{\nut^2+\nuo^2}{C_{10}}}\pts{\ndpnd} + \hc_9 \nuo^2\pts{\ndmnd}-C_{12} \sqrt{n\nd}\pts{\nut+\nuo}}^2}{\hc_{13} \nd d}.\label{eq:maj_target_primitive_1}
\end{align}
Next, since $R_- \geq 0$, i.e.~$\nut \geq \nuo$, we get
\begin{align}
    \eqref{eq:maj_target_primitive_1} &\geq \dfrac{\pts{\pts{\hc_9\nut^2 -\frac{2\nut^2}{C_{10}}}\pts{\ndpnd} + \hc_9 \nuo^2\pts{\ndmnd}-2C_{12} \sqrt{n\nd}\nut}^2}{\hc_{13} \nd d}\nonumber\\
    &\geq \dfrac{\pts{\pts{\hc_9 -\frac{2}{C_{10}} -\frac{C_{12}}{C_{14}}}\nut^2\pts{\ndpnd} + \hc_9 \nuo^2\pts{\ndmnd}}^2}{\hc_{13} \nd d}\nonumber\\
    &\geq \dfrac{\pts{\hc_{15}\nut^2\pts{\ndpnd} + \hc_{15} \nuo^2\pts{\ndmnd}}^2}{\hc_{13} \nd d}\nonumber\\
    &= \dfrac{\pts{\hc_{15} R_+\frac{n_+}{\Delta_+} + \hc_{15} R_-\frac{n_-}{\Delta_-}}^2}{\hc_{13} \nd d} \geq \dfrac{\hc_{15}^2 R_+^2\frac{n_+^2}{\Delta_+^2} + \hc_{15}^2 R_-^2\frac{n_-^2}{\Delta_-^2}}{\hc_{13}\nd d},\label{eq:maj_target_primitive_2}
\end{align}
where the second inequality follows by Lemma~\ref{lem:aux_compare_a_h} with a large enough constant $C_{14} > C_{12}$.
Finally, we define $\alpha_\pm = \frac{n_\pm/\Delta_\pm^2}{\nd}$, noting that $0<\alpha_\pm<1$ and $\alpha_+ + \alpha_-=1$. Eq.~\eqref{eq:maj_target_primitive_2} becomes
\begin{align*}
    \frac{\pts{\hw^\top\bmu_{+1}}^2}{2\hw^{\top}\hw} &\geq \frac{\hc_{16} \pts{\alpha_+R_+^2n_+ + \alpha_-R_-^2n_-}}{d}.
\end{align*}
Plugging the above into Lemma~\ref{lem:wst_risk_q_upper} completes the proof.
\qed 

\section{Conclusion}\label{sec:conclusion}
In this work, we provide a sharp rate on the worst-group error of cost-sensitive interpolating or ridge-regularized classifiers as well as SVM classifiers, and further show that this rate is tight up to universal constants. 
Our rate explicitly characterizes the role of the adjustment weights $\Delta_{\pm}$, through which we identify a new robustness-accuracy tradeoff for this class of estimators.
An intriguing question that remains is whether this tradeoff is information-theoretically optimal under this data model in the overparameterized regime, or whether an estimator with an inductive bias that is fundamentally different from the $\ell_2$-norm might improve it.

\section*{Acknowledgements}
We gratefully acknowledge the support of the NSF (through CAREER award CCF-2239151 and award IIS2212182), an Adobe Data Science Research Award, an Amazon Research Award and a Google Research Colabs award.

\printbibliography
\newpage
\appendix

\section{Proofs of primitive lemmas (Lemmas~\ref{lem:primitive_bounds} and~\ref{lem:w_primitive_bounds})} \label{app:primitive_proof}
In this section, we present the proofs of bounds on the primitives. There are two essential techniques that we use throughout these proofs. 
First, in a manner similar to~\cite{wang2021benign}, we use the recursive Woodbury identity to characterize the Gram matrix inverse $(\XX+\tau\I)^{-1}$ (and, thereby, the second-order primitives) in terms of the zero-th order primitives involving the inverse Wishart matrix $\QQtaui$.
We note that, unlike in~\cite{wang2021benign} (which assumed orthogonal label vectors across classes), the label vector and the attribute vector \emph{are not} orthogonal here; therefore, a new technique is required to characterize the second-order primitives in terms of zeroth order primitives up to multiplicative constants.
Second, due to the presence of the label vector $\y$ and attribute vector $\ba$, we are required to bound several ``cross terms" in the primitives, for which naive variational characterizations of the inverse Wishart matrix $\QQtaui$ will not suffice.
The analysis uses concentration bounds on quadratic forms of inverse Wishart matrices that are sharp even in their multiplicative constant, facilitating a tight analysis of these cross terms.
Finally, special care is taken to sharply characterize the role of the label adjustment matrix $\bDelta^{-1}$ in the analysis.

To begin with, we recall the definition of $\XXk$ and $\XXktaui$ in Eq.~\eqref{eq:xx_def} and Eq.~\eqref{eq:xxi_def} for $k=0,1,2$.
%
Meanwhile, we write $\A_k$ in terms of the primitives as below:
\begin{gather*}
    \A_k=\begin{bmatrix}
        1+\nuk^2 \skkkmo & \nuk \hkkkmo & \nuk \skkkmo\\
        \nuk\skkkmo & 1 + \hkkkmo & \skkkmo\\
        \nuk\hkkkmo & \tkkkmo & 1+ \hkkkmo
    \end{bmatrix},
\end{gather*}
its determinant as
\begin{gather}\label{eq:det_def}
    \det(\A_k)=\skkkmo\pts{\nuk^2-\tkkkmo}+\pts{\hkkkmo+1}^2,
\end{gather}
and its adjugate matrix $\text{adj}(\A_k)$ in column format as
\begin{gather*}
    \text{adj}(\A_k)_{1}=\begin{bmatrix}
        \pts{\hkkkmo+1}^2-\skkkmo\tkkkmo \\
        -\nuk\skkkmo\\
        \nuk\pts{\skkkmo\tkkkmo-\hkkkmo-{\hkkkmo}^2}
    \end{bmatrix},\\
    \text{adj}(\A_k)_{2}=\begin{bmatrix}
        \nuk\pts{\skkkmo\tkkkmo-\hkkkmo-{\hkkkmo}^2}\\
        \hkkkmo+1+\nuk^2\skkkmo\\
        \nuk^2{\hkkkmo}^2-\tkkkmo\pts{1+\nuk^2\skkkmo}
    \end{bmatrix},\\
    \text{adj}(\A_k)_{3}=\begin{bmatrix}
        -\nuk\skkkmo\\
        -\skkkmo\\
        \hkkkmo+1+\nuk^2\skkkmo
    \end{bmatrix}.
\end{gather*}
Next, we will need the following inverse Wishart concentration lemma from~\cite{muthukumar2021classification}.
%
%
\begin{lemma}[\cite{muthukumar2021classification}, Lemma 21]\label{lem:wishart_bound}
    Let $\A \sim \text{Wishart}(d, \I_n)$ and $d'(n) \coloneqq d - n + 1$. For any vector $\bu \in S^{n-1}$ and any $t > 0$, we have
    \begin{gather}
        \text{Pr}\mts{\frac{1}{\bu^{\top}\A^{-1}\bu} > d'(n) + \sqrt{2t \cdot d'(n)} + 2t} \leq e^{-t}\label{eq:wishart_upper}\\
        \text{Pr}\mts{\frac{1}{\bu^{\top}\A^{-1}\bu} < d'(n) - \sqrt{2t \cdot d'(n)}} \leq e^{-t}\label{eq:wishart_lower},
    \end{gather}
    provided that $d'(n) > 2 \text{max}\{t, 1\}$.
\end{lemma}
The following corollary further characterizes the eigenvalues of an inverse ridge-regularized inverse Wishart matrix.
\begin{corollary}\label{cor:wishart_bound_ridge}
    Let $\A \sim \text{Wishart}(d, \I_n)$ and $d'(n) \coloneqq d - n + 1$. For any $t > 0$, and $\tau \geq 0$, we have
    \begin{gather}
        \frac{1}{d'(n) + \sqrt{2t \cdot d'(n)} + 2t + \tau} \leq \lbf[n]{\pts{\A +\tau\I}^{-1}} \leq \lbf[1]{\pts{\A +\tau\I}^{-1}} \leq \frac{1}{d'(n) - \sqrt{2t \cdot d'(n)} + \tau},
    \end{gather}
    with probability at least $1-e^{-t}$, provided that $d'(n) > 2 \text{max}\{t, 1\}$.
\end{corollary}
\begin{proof}
    We first show the upper bound
    \begin{align*}
        \lbf[1]{\pts{\A +\tau\I}^{-1}} = \frac{1}{\lbf[n]{\A +\tau\I}} = \frac{1}{\lbf[n]{\A} + \tau} = \frac{1}{\frac{1}{\lbf[1]{\A^{-1}}} + \tau} \leq \frac{1}{d'(n) - \sqrt{2t \cdot d'(n)} + \tau},
    \end{align*}
    where the inequality follows by Eq.~\eqref{eq:wishart_lower} with probability at least $1-e^{-t}$. Similarly, for the lower bound, we have
    \begin{align*}
        \lbf[n]{\pts{\A +\tau\I}^{-1}} &= \frac{1}{\lbf[1]{\A +\tau\I}} = \frac{1}{\lbf[1]{\A} + \tau} = \frac{1}{\frac{1}{\lbf[n]{\A^{-1}}} + \tau}  \geq \frac{1}{d'(n) + \sqrt{2t \cdot d'(n)} + 2t + \tau},
    \end{align*}
    where the inequality follows by Eq.~\eqref{eq:wishart_upper} with probability at least $1-e^{-t}$.
     
\end{proof}

\subsection{Proof for Lemma~\ref{lem:primitive_bounds} (without adjustment weight)}
\begin{proof}(Lemma~\ref{lem:primitive_bounds})
    We prove the lemma by induction. We first show the bounds for the order $k=0$ (base case). Then, we assume that the bounds are true for order $k=k$ and we show the bounds for order $k=k+1$. The proof for the base case $k=0$ resembles the primitive proofs in~\cite{wang2021benign}.
    Note that the indexing of constants in the proof below does not match the indexing of constants in the main statement of Lemma~\ref{lem:primitive_bounds}.
    \\
    \textbf{Bounds for order $k=0$.}\\
    We have $\bullet\; \siiz = \bv_i^\top\XXztaui\bv_i = \bv_i^\top\QQtaui\bv_i$ for $i=1,2$. Based on Corollary~\ref{cor:wishart_bound_ridge}, let $t=\sqrt{2\log\pts{n/\delta}}$. Then, we have with probability at least $1-\frac{\delta^2}{n^2}$, 
    \begin{align*}
        \siiz  \leq \frac{\nvi^2}{d'(n) - \sqrt{4\log\pts{n/\delta} \cdot d'(n)} + \tau} \leq \frac{n}{d'(n)\pts{1-\frac{2}{\sqrt{C}n}} + \tau} = \frac{n\pts{\frac{\sqrt{C}n}{\sqrt{C}n-2}}}{d'(n) + \pts{\frac{\sqrt{C}n}{\sqrt{C}n-2}}\tau}  &\leq \pts{\frac{C_1+1}{C_1}}\frac{n}{d+\tau},
    \end{align*}
    where the second inequality follows by Assumption~\ref{asm:dataset_gen}(D) such that $\frac{1}{Cn^2} \geq \frac{\log(n/\delta)}{d'(n)}$. Similarly, 
    \begin{align*}
        \siiz  \geq \frac{\nvi^2}{d'(n) + \sqrt{4\log\pts{n/\delta} \cdot d'(n)} + 4\log\pts{n/\delta}+\tau} \geq \frac{n}{d'(n)\pts{1+\frac{2}{\sqrt{C}n}+\frac{4}{Cn}}+\tau} &\geq \pts{\frac{C_1-1}{C_1}}\frac{n}{d+\tau}.
    \end{align*}
    %
    %
    $\bullet\;\sotz = \bv_1^\top\XXztaui\bv_2 = \y^\top\QQtaui\ba$. Note that this is a ``cross-term" primitive.
    We apply the parallelogram law to get
    \begin{align*}
        \sotz &= \y^\top\QQtaui\ba =\frac{1}{4}\pts{\bv_+^\top\QQtaui\bv_+ - \bv_-^\top\QQtaui\bv_-},
    \end{align*}
    where we denote $\bv_\pm \coloneqq \y\pm\ba$. It is easy to verify that $\nnorm[2]{\bv_+}^2=4n_+$ and $\nnorm[2]{\bv_-}^2=4n_-$. Hence, again applying Corollary~\ref{cor:wishart_bound_ridge} gives us
    \begin{align*}
        \sotz &\leq\frac{1}{4}\pts{\frac{4n_+}{d'(n) - \sqrt{4\log\pts{n/\delta} \cdot d'(n)}+\tau} - \frac{4n_-}{d'(n) + \sqrt{4\log\pts{n/\delta} \cdot d'(n)} + 4\log\pts{n/\delta}+\tau}}\\
        &\leq \frac{n_+}{d'(n)\pts{1-\frac{2}{\sqrt{Cn^2}}}+\tau} - \frac{n_-}{d'(n)\pts{1+\frac{2}{\sqrt{Cn^2}}+\frac{4}{Cn^2}}+\tau}\\
        &\leq \pts{\frac{C_2+1}{C_2}}\frac{n_+}{d+\tau}-\pts{\frac{C_2-1}{C_2}}\frac{n_-}{d+\tau} =: \hc_3 \frac{n_+}{d+\tau} - \dc_3 \frac{n_-}{d + \tau}. \\
        \sotz &\geq \frac{1}{4}\pts{\frac{4n_+}{d'(n) + \sqrt{4\log\pts{n/\delta} \cdot d'(n)} + 4\log\pts{n/\delta}+\tau} - \frac{4n_-}{d'(n) - \sqrt{4\log\pts{n/\delta} \cdot d'(n)}+\tau}}\\
        &\geq \frac{n_+}{d'(n)\pts{1+\frac{2}{\sqrt{Cn^2}} +\frac{4}{Cn^2}}+\tau} - \frac{n_-}{d'(n)\pts{1-\frac{2}{\sqrt{Cn^2}}}+\tau}\\
        &\geq \pts{\frac{C_2-1}{C_2}}\frac{n_+}{d+\tau}-\pts{\frac{C_2+1}{C_2}}\frac{n_-}{d+\tau}  =: \dc_3 \frac{n_+}{d+\tau} - \hc_3 \frac{n_-}{d+\tau}.
    \end{align*}
    %
    $\bullet\; \tiiz = \bd_i^\top\XXztaui\bd_i = \bd_i^\top\QQtaui\bd_i,$ for $i=1,2$. Note that the lemma in~\cite{wang2022binary}[Corollary 8.1] gives $\ndi\leq C\sqrt{n}\nui$, we have
    \begin{align*}
        \tiiz &\leq \frac{\ndi^2}{d'(n) - \sqrt{4\log\pts{n/\delta} \cdot d'(n)}+\tau} \leq \frac{Cn\nui^2}{d'(n)\pts{1-\frac{2}{\sqrt{Cn^2}}}+\tau} \leq \frac{C_3n\nui^2}{d+\tau}.
    \end{align*}
    %
    $\bullet \; \totz = \bd_1^\top\XXztaui\bd_2 = \bd_1^\top\QQtaui\bd_2.$ By the sub-multiplicative property of the matrix norm, we get
    \begin{align*}
        \totz &\leq \ndo\ndt\nnorm[2]{\QQtaui} \leq \frac{Cn\nuo\nut}{d'(n) - \sqrt{4\log\pts{n/\delta} \cdot d'(n)}+\tau} \leq \frac{C_3n\nuo\nut}{d+\tau}.\\
        %
        \totz &\geq -\ndo\ndt\nnorm[2]{\QQtaui} \geq \frac{-C_3n\nuo\nut}{d+\tau}.
    \end{align*}
    $\bullet \; \hijz = \bd_i^\top\XXztaui\bv_j = \bd_i^\top\QQtaui\bv_j$, for $i,j=1,2$. Hence,
    \begin{align*}
        \hijz &\leq \ndi\nvj\nnorm[2]{\QQtaui} \leq \frac{Cn\nui}{d'(n) - \sqrt{4\log\pts{n/\delta} \cdot d'(n)}+\tau} \leq \frac{C_4n\nui}{d+\tau}.\\
        %
        \hijz &\geq -\ndi\nvj\nnorm[2]{\QQtaui} \geq \frac{-C_4n\nui}{d+\tau}.
    \end{align*}
    %
    %
    $\bullet\; \suuz = \bu^\top\XXztaui\bu = \bu^\top\QQtaui\bu$. By Corollary~\ref{cor:wishart_bound_ridge} and setting $t=\sqrt{2\log\pts{n/\delta}}$, we have with probability at least $1-\frac{\delta^2}{n^2}$, 
    \begin{align*}
        \suuz & \leq \frac{1}{d'(n) - \sqrt{4\log\pts{n/\delta} \cdot d'(n)}+\tau} \leq \frac{1}{d'(n)\pts{1-\frac{2}{\sqrt{C}n}}+\tau} \leq \pts{\frac{C_5+1}{C_5}}\frac{1}{d+\tau},\\
        \suuz & \geq \frac{1}{d'(n) + \sqrt{4\log\pts{n/\delta} \cdot d'(n)} + 4\log\pts{1/\delta}+\tau} \geq \frac{n}{d'(n)\pts{1+\frac{2}{\sqrt{C}n}+\frac{4}{Cn}}+\tau} \geq \pts{\frac{C_5-1}{C_5}}\frac{1}{d+\tau}.
    \end{align*}
    %
    $\bullet\;\suiz = \bu^\top\XXztaui\bu = \bu^\top\QQtaui\bu$ for $i=1,2$. Again, by Corollary~\ref{cor:wishart_bound_ridge}, we have 
    \begin{align*}
        \suiz &\leq \nbu\nvi\nnorm[2]{\QQtaui} \leq \frac{C\sqrt{n}}{d'(n) - \sqrt{4\log\pts{n/\delta} \cdot d'(n)}+\tau} \leq \frac{C_6\sqrt{n}}{d+\tau}.\\
        %
        \suiz &\geq -\nbu\nvi\nnorm[2]{\QQtaui} \geq \frac{-C_6\sqrt{n}}{d+\tau}.
    \end{align*}
    %
    $\bullet \; \hiuz = \bd_i^\top\XXztaui\bu = \bd_i^\top\QQtaui\bu$, for $i=1,2$. Hence,
    \begin{align*}
        \hiuz &\leq \ndi\nbu\nnorm[2]{\QQtaui} \leq \frac{C\sqrt{n}\nui}{d'(n) - \sqrt{4\log\pts{n/\delta} \cdot d'(n)}+\tau} \leq \frac{C_7\sqrt{n}\nui}{d+\tau}.\\
        %
        \hiuz &\geq -\ndi\nbu\nnorm[2]{\QQtaui} \geq \frac{-C_7\sqrt{n}\nui}{d+\tau}.
    \end{align*}
    
    Thus, we have characterized all of the zeroth-order primitives.
    Next, we assume that the bounds for order $k=k$ are true, and we show the bounds for order $k=k+1$. The key idea is that the higher order primitives are a constant fraction of the previous-order primitives. The following lemma, whose proof is deferred to Appendix~\ref{app:aux_lem_proof}, details the recursive calculation that shows this.
    \begin{lemma} \label{lem:aux_f_func}
        For a function $f_{A_k}:\R^4\rightarrow\R$ such that
        \begin{align*}
            \ffs[\A_{k}]{\begin{matrix}
                x_a, x_b,\\ x_c, x_d
            \end{matrix}} &\coloneqq \mts{\nuk x_a,~x_b,~x_a}\text{adj}(\A_{k})\begin{bmatrix}
            \nuk x_c \\
            x_c \\
            x_d
            \end{bmatrix}\\
            &= \pts{\nuk^2-\tkkkmo}x_a x_c + \pts{1+\hkkkmo}\pts{x_a x_d + x_b x_c} -\skkkmo x_b x_d,
        \end{align*}
        we have the absolute value upper bound
        \begin{align*}
            \normo{\ffs[\A_{k}]{\begin{matrix}
                x_a, x_b,\\ x_c, x_d
            \end{matrix}}} & \leq \nuk^2 \normo{x_a} \normo{x_c} + \normo{\skkkmo}\normo{x_b} \normo{x_d} + \pts{1+\frac{1}{C_1}}\pts{\normo{x_a} \normo{x_d} + \normo{x_b} \normo{x_c}}.
        \end{align*}
    \end{lemma}
    The following lemma provides an upper and lower bound for $\det(\A_k)$. The proof of this lemma can be found in Appendix~\ref{app:aux_lem_proof}.
    \begin{lemma}\label{lem:aux_det_bound}
        For $k = 1, 2$, we have
        \begin{gather*}
            \dc_1 \leq \det(\A_k) \leq \hc_1,
        \end{gather*}
        for some $C_1>1$ and $\hc_m \coloneqq \frac{C_m+1}{C_m}$ and $\dc_m \coloneqq \frac{C_m-1}{C_m}$.
    \end{lemma}
    We can now complete the proof of Lemma~\ref{lem:primitive_bounds}.
    We only present the calculation of primitive $\siikpo$ due to page limitations; the proofs of the other primitives follow the same idea and are presented in Appendix~\ref{app:primitive_recursive_proof}.\\
    \textbf{Bounds for order $k=k+1$}\\
    $\bullet \; \siikpo$ for $i = 1,2 $.
    \begin{align*}
        &\siikpo = \bv_i^\top\XXkpotaui\bv_i\\
        &=\bv_i^\top\XXktaui\bv_i - \bv_i^\top\XXktaui\bL_{k+1}\A_{k+1}^{-1}\bR_{k+1}\XXktaui\bv_i\\
        &= \siik - \mts{\nukpo\sikpok,~\hkpoik,~\sikpok} \times\frac{\text{adj}(\A_{k+1})}{\det(\A_{k+1})}\begin{bmatrix}
        \nukpo\sikpok \\
        \sikpok \\
        \hkpoik
        \end{bmatrix}\\
        &= \frac{\siik \det(\A_{k+1}) - \ffs[\A_{k+1}]{\sikpok, \hkpoik, \sikpok, \hkpoik}}{\det(\A_{k+1})}.
    \end{align*}
    Thus, we can apply Lemma~\ref{lem:aux_f_func} to get 
    \begin{align*}
        \siikpo &\leq \frac{\normo{\siik \det(\A_{k+1})} + \normo{\ffs[\A_{k+1}]{\sikpok, \hkpoik, \sikpok, \hkpoik}}}{\det(\A_{k+1})}\\
        &\leq \frac{1}{\det(\A_{k+1})}\Bigg(\det(\A_{k+1})\normo{\siik} + \nukpo^2\normo{\sikpok}\normo{\sikpok}\\
        &+ \pts{1+\frac{1}{C_1}}\pts{\normo{\sikpok}\normo{\hkpoik} + \normo{\hkpoik}\normo{\sikpok}} + \normo{\skpokpok}\normo{\hkpoik}\normo{\hkpoik}\Bigg)\\
        &\leq \frac{1}{\det(\A_{k+1})} \pts{\frac{C_2+1}{C_2}\frac{n}{d+\tau} + \frac{C_3+1}{C_3}\frac{\nukpo^2n^2}{\pts{d+\tau}^2} + \frac{C_4+1}{C_4}\frac{C_5\nukpo n^2}{\pts{d+\tau}^2} + \frac{C_6\nukpo^2n^3}{\pts{d+\tau}^3}}\\
        & \leq \frac{C_7+1}{C_7}\frac{n}{d+\tau},
    \end{align*}
    where in the third inequality we apply lower bound of $\det(\A_{k+1})$ in Lemma~\ref{lem:aux_det_bound} and primitive bounds in order $k$, and in the last inequality, we apply Assumption~\ref{asm:dataset_gen}(C) such that $d \geq CR_+n$.  Similarly, 
    \begin{align*}
        \siikpo &\geq \frac{\siik \det(\A_{k+1}) - \normo{\ffs[\A_{k+1}]{\sikpok, \hkpoik, \sikpok, \hkpoik}}}{\det(\A_{k+1})}\\
        &\geq \frac{1}{\det(\A_{k+1})}\Bigg(\det(\A_{k+1})\siik- \Bigg(\nukpo^2\normo{\sikpok}\normo{\sikpok}\\
        &+ \pts{1+\frac{1}{C_1}}\pts{\normo{\sikpok}\normo{\hkpoik} + \normo{\hkpoik}\normo{\sikpok}} + \normo{\skpokpok}\normo{\hkpoik}\normo{\hkpoik}\Bigg)\Bigg)\\
        &\geq \frac{1}{\det(\A_{k+1})} \pts{\frac{C_2-1}{C_2}\frac{n}{d+\tau}-\frac{C_3+1}{C_3}\frac{\nukpo^2n^2}{\pts{d+\tau}^2} - \frac{C_4+1}{C_4}\frac{C_5\nukpo n^2}{\pts{d+\tau}^2} - \frac{C_6\nukpo^2n^3}{\pts{d+\tau}^3} }\\
        &\geq \frac{C_7-1}{C_7}\frac{n}{d+\tau}.
    \end{align*}
    This completes the proof of the lemma. 
\end{proof}
\subsection{Proof for Lemma~\ref{lem:w_primitive_bounds} (with adjustment weights)}
\begin{proof}(Lemma~\ref{lem:w_primitive_bounds})
    Before we start the proof of Lemma~\ref{lem:w_primitive_bounds}, we need the following auxiliary lemma, whose proof we defer to Appendix~\ref{app:aux_lem_proof}.
    \begin{lemma}\label{lem:aux_ineq}
        For $\frac{1}{n}\leq\Delta_{\pm} \leq 1$ and $\Delta_+ \geq \Delta_-$, we have
        \begin{gather*}
            \tilde{C}\mts{\frac{1}{\sqrt{C}n}(n + \frac{n_+}{\Delta_+^2} + \frac{n_-}{\Delta_-^2})} \leq \frac{1}{2}\pts{\ndpnd}.
        \end{gather*}
        for some $C>144$, and $\tilde{C}>2$.
    \end{lemma}
\noindent Similar to the proof of Lemma~\ref{lem:primitive_bounds}, we prove by induction.
    
\noindent \textbf{Bounds for base case $k=0$:}\\
    $\bullet \; \sidiz = \bv_i^\top\bDelta^{-1}\XXztaui\bv_i = \bv_i^\top\bDelta^{-1}\QQtaui\bv_i,$ for $i=1,2$. Applying the parallelogram law gives
    \begin{align*}
        &\sidiz\\ 
        &= \frac{1}{4}\pts{\pts{\bDelta^{-1}\bv_i+\bv_i}^\top\QQtaui\pts{\bDelta^{-1}\bv_i+\bv_i} -\pts{\bDelta^{-1}\bv_i-\bv_i}^\top\QQtaui\pts{\bDelta^{-1}\bv_i-\bv_i}}.
    \end{align*}
    Next, we have
    \begin{align*}
        \nnorm[2]{\bDelta^{-1}\bv_i+\bv_i}^2 &= \pts{1+\frac{1}{\Delta_+}}^2n_+ + \pts{1+\frac{1}{\Delta_-}}^2n_- = n + \frac{2n_+}{\Delta_+} + \frac{n_+}{\Delta_+^2} + \frac{2n_-}{\Delta_-} + \frac{n_-}{\Delta_-^2},\\
        \nnorm[2]{\bDelta^{-1}\bv_i-\bv_i}^2 &= \pts{1-\frac{1}{\Delta_+}}^2n_+ + \pts{1-\frac{1}{\Delta_-}}^2n_- = n - \frac{2n_+}{\Delta_+} + \frac{n_+}{\Delta_+^2} - \frac{2n_-}{\Delta_-} + \frac{n_-}{\Delta_-^2}.
    \end{align*}
    Applying Corollary~\ref{cor:wishart_bound_ridge} with $t=2\log\pts{n/\delta}$, we have, with probability at least $1-\frac{\delta^2}{n^2}$,
    \begin{align*}
        \sidiz &\leq\frac{1}{4}\pts{\frac{n + \frac{2n_+}{\Delta_+} + \frac{n_+}{\Delta_+^2} + \frac{2n_-}{\Delta_-} + \frac{n_-}{\Delta_-^2}}{d'(n) - \sqrt{4\log\pts{n/\delta} \cdot d'(n)}+\tau} - \frac{n - \frac{2n_+}{\Delta_+} + \frac{n_+}{\Delta_+^2} - \frac{2n_-}{\Delta_-} + \frac{n_-}{\Delta_-^2}}{d'(n) + \sqrt{4\log\pts{n/\delta} \cdot d'(n)} + 4\log\pts{n/\delta}+\tau}}\\
        &\leq \frac{1}{4}\pts{\frac{n + \frac{2n_+}{\Delta_+} + \frac{n_+}{\Delta_+^2} + \frac{2n_-}{\Delta_-} + \frac{n_-}{\Delta_-^2}}{d'(n)\pts{1-\frac{2}{\sqrt{C}n}}+\tau} - \frac{n - \frac{2n_+}{\Delta_+} + \frac{n_+}{\Delta_+^2} - \frac{2n_-}{\Delta_-} + \frac{n_-}{\Delta_-^2}}{d'(n)\pts{1+\frac{2}{\sqrt{C}n}+\frac{4}{Cn^2}}+\tau}}\\
        &\leq \frac{\pts{1+\frac{2}{Cn^2}}\pts{\frac{n_+}{\Delta_+}+\frac{n_-}{\Delta_-}} + \pts{\frac{1}{\sqrt{C}n} +\frac{1}{Cn^2}}\pts{n+\frac{n_+}{\Delta_+^2}+\frac{n_-}{\Delta_-^2}}}{d'(n)\pts{1-\frac{2}{\sqrt{C}n}}\pts{1+\frac{2}{\sqrt{C}n}+\frac{4}{Cn^2}} +\tau}\\
        &\leq \frac{\pts{1+\frac{2}{Cn^2}}\pts{\frac{n_+}{\Delta_+}+\frac{n_-}{\Delta_-}}+ \frac{2}{\sqrt{C}n}\pts{n+\frac{n_+}{\Delta_+^2}+\frac{n_-}{\Delta_-^2}}}{d +\tau}\\
        &\leq \frac{\pts{1+\frac{2}{Cn^2} +\frac{1}{\tilde{C}}}\pts{\frac{n_+}{\Delta_+}+\frac{n_-}{\Delta_-}}}{d +\tau} =\frac{\pts{\frac{C_1+1}{C_1} + \frac{1}{\tilde{C}}}\pts{\frac{n_+}{\Delta_+}+\frac{n_-}{\Delta_-}}}{d +\tau},
    \end{align*}
    where the second inequality follows by Assumption~\ref{asm:dataset_gen}(D) such that $d'(n) > Cn^2\log(n/\delta)$, and the last inequality follows by Lemma~\ref{lem:aux_ineq}. Next, we have the lower bound:
    \begin{align*}
        \sidiz &\geq \frac{1}{4}\pts{\frac{n + \frac{2n_+}{\Delta_+} + \frac{n_+}{\Delta_+^2} + \frac{2n_-}{\Delta_-} + \frac{n_-}{\Delta_-^2}}{d'(n) + \sqrt{4\log\pts{n/\delta} \cdot d'(n)} + 4\log\pts{n/\delta} +\tau} - \frac{n - \frac{2n_+}{\Delta_+} + \frac{n_+}{\Delta_+^2} - \frac{2n_-}{\Delta_-} + \frac{n_-}{\Delta_-^2}}{d'(n) - \sqrt{4\log\pts{n/\delta} \cdot d'(n)} +\tau}}\\
        &\geq \frac{1}{4}\pts{\frac{n + \frac{2n_+}{\Delta_+} + \frac{n_+}{\Delta_+^2} + \frac{2n_-}{\Delta_-} + \frac{n_-}{\Delta_-^2}}{d'(n)\pts{1+\frac{2}{\sqrt{C}n}+\frac{4}{Cn^2}} +\tau} - \frac{n - \frac{2n_+}{\Delta_+} + \frac{n_+}{\Delta_+^2} - \frac{2n_-}{\Delta_-} + \frac{n_-}{\Delta_-^2}}{d'(n)\pts{1-\frac{2}{\sqrt{C}n}} +\tau}}\\
        &\geq \frac{\pts{1+\frac{2}{Cn^2}}\pts{\frac{n_+}{\Delta_+}+\frac{n_-}{\Delta_-}} - \pts{\frac{1}{\sqrt{C}n} +\frac{1}{Cn^2}}\pts{n+\frac{n_+}{\Delta_+^2}+\frac{n_-}{\Delta_-^2}}}{d'(n)\pts{1-\frac{2}{\sqrt{C}n}}\pts{1+\frac{2}{\sqrt{C}n}+\frac{4}{Cn^2}} +\tau}\\
        &\geq \frac{\pts{1+\frac{2}{Cn^2}}\pts{\frac{n_+}{\Delta_+}+\frac{n_-}{\Delta_-}} - \frac{2}{\sqrt{C}n}\pts{n+\frac{n_+}{\Delta_+^2}+\frac{n_-}{\Delta_-^2}}}{d +\tau}\\
        &\geq \frac{\pts{1+\frac{2}{Cn^2}-\frac{1}{\tilde{C}}}\pts{\frac{n_+}{\Delta_+}+\frac{n_-}{\Delta_-}}}{d +\tau} = \frac{\pts{\frac{C_1-1}{C_1} - \frac{1}{\tilde{C}}}\pts{\frac{n_+}{\Delta_+}+\frac{n_-}{\Delta_-}}}{d +\tau},
    \end{align*}
    where the second inequality follows by Assumption~\ref{asm:dataset_gen}(D) such that $d'(n) > Cn^2\log(n/\delta)$, and the fourth inequality follows by Lemma~\ref{lem:aux_ineq}.\\
    $\bullet \; \stdoz  = \bv_2^\top\bDelta^{-1}\XXztaui\bv_1 = \y^\top\bDelta^{-1}\QQtaui\ba.$ Applying the parallelogram law gives
        \begin{align*}
            \stdoz = \frac{1}{4}\pts{\pts{\bDelta^{-1}\y+\ba}^\top\QQtaui\pts{\bDelta^{-1}\y+\ba} -\pts{\bDelta^{-1}\y-\ba}^\top\QQtaui\pts{\bDelta^{-1}\y-\ba}}.
        \end{align*}
        Next, we have
        \begin{align*}
            \nnorm[2]{\bDelta^{-1}\y+\ba}^2 &= \pts{1+\frac{1}{\Delta_+}}^2n_+ + \pts{1-\frac{1}{\Delta_-}}^2n_- = n + \frac{2n_+}{\Delta_+} + \frac{n_+}{\Delta_+^2} - \frac{2n_-}{\Delta_-} + \frac{n_-}{\Delta_-^2}.\\
            \nnorm[2]{\bDelta^{-1}\y-\ba}^2 &= \pts{1-\frac{1}{\Delta_+}}^2n_+ + \pts{1+\frac{1}{\Delta_-}}^2n_- = n - \frac{2n_+}{\Delta_+} + \frac{n_+}{\Delta_+^2} + \frac{2n_-}{\Delta_-} + \frac{n_-}{\Delta_-^2}.
        \end{align*}
        Therefore, we get
        \begin{align*}
            \stdoz &\leq \frac{1}{4}\pts{\frac{n + \frac{2n_+}{\Delta_+} + \frac{n_+}{\Delta_+^2} - \frac{2n_-}{\Delta_-} + \frac{n_-}{\Delta_-^2}}{d'(n) - \sqrt{4\log\pts{n/\delta} \cdot d'(n)} +\tau} - \frac{n - \frac{2n_+}{\Delta_+} + \frac{n_+}{\Delta_+^2} + \frac{2n_-}{\Delta_-} + \frac{n_-}{\Delta_-^2}}{d'(n) + \sqrt{4\log\pts{n/\delta} \cdot d'(n)} + 4\log\pts{n/\delta} +\tau}}\\
            &\leq \frac{1}{4}\pts{\frac{n + \frac{2n_+}{\Delta_+} + \frac{n_+}{\Delta_+^2} - \frac{2n_-}{\Delta_-} + \frac{n_-}{\Delta_-^2}}{d'(n)\pts{1-\frac{2}{\sqrt{C}n}} +\tau} - \frac{n - \frac{2n_+}{\Delta_+} + \frac{n_+}{\Delta_+^2} + \frac{2n_-}{\Delta_-} + \frac{n_-}{\Delta_-^2}}{d'(n)\pts{1+\frac{2}{\sqrt{C}n}+\frac{4}{Cn^2}} +\tau}}\\
            &\leq \frac{\pts{1+\frac{2}{Cn^2}}\pts{\frac{n_+}{\Delta_+}-\frac{n_-}{\Delta_-}}  + \pts{\frac{1}{\sqrt{C}n} +\frac{1}{Cn^2}}\pts{n + \frac{n_+}{\Delta_+^2}+\frac{n_-}{\Delta_-^2}}}{d'(n)\pts{1-\frac{2}{\sqrt{C}n}}\pts{1+\frac{2}{\sqrt{C}n}+\frac{4}{Cn^2}} +\tau}\\
            &\leq \frac{\pts{1+\frac{2}{Cn^2}}\pts{\frac{n_+}{\Delta_+}-\frac{n_-}{\Delta_-}} + \frac{2}{\sqrt{C}n}\pts{n + \frac{n_+}{\Delta_+^2}+\frac{n_-}{\Delta_-^2}}}{d +\tau}\\
            &\leq \frac{\pts{1+\frac{2}{Cn^2}}\pts{\frac{n_+}{\Delta_+}-\frac{n_-}{\Delta_-}} + \frac{1}{\tilde{C}}\pts{\frac{n_+}{\Delta_+}+\frac{n_-}{\Delta_-}}}{d +\tau} = \frac{\pts{\frac{C_1+1}{C_1}}\pts{\frac{n_+}{\Delta_+}-\frac{n_-}{\Delta_-}} + \frac{1}{\tilde{C}}\pts{\frac{n_+}{\Delta_+}+\frac{n_-}{\Delta_-}}}{d +\tau}.
        \end{align*}
        \begin{align*}
            \stdoz &\geq \frac{1}{4}\pts{\frac{n + \frac{2n_+}{\Delta_+} + \frac{n_+}{\Delta_+^2} - \frac{2n_-}{\Delta_-} + \frac{n_-}{\Delta_-^2}}{d'(n) + \sqrt{4\log\pts{n/\delta} \cdot d'(n)} + 4\log\pts{n/\delta} +\tau} - \frac{n - \frac{2n_+}{\Delta_+} + \frac{n_+}{\Delta_+^2} + \frac{2n_-}{\Delta_-} + \frac{n_-}{\Delta_-^2}}{d'(n) - \sqrt{4\log\pts{n/\delta} \cdot d'(n)} +\tau}}\\
            &\geq \frac{1}{4}\pts{\frac{n + \frac{2n_+}{\Delta_+} + \frac{n_+}{\Delta_+^2} - \frac{2n_-}{\Delta_-} + \frac{n_-}{\Delta_-^2}}{d'(n)\pts{1+\frac{2}{\sqrt{C}n}+\frac{4}{Cn^2}} +\tau} - \frac{n - \frac{2n_+}{\Delta_+} + \frac{n_+}{\Delta_+^2} + \frac{2n_-}{\Delta_-} + \frac{n_-}{\Delta_-^2}}{d'(n)\pts{1-\frac{2}{\sqrt{C}n}} +\tau}}\\
            &\geq \frac{\pts{1+\frac{2}{Cn^2}}\pts{\frac{n_+}{\Delta_+}-\frac{n_-}{\Delta_-}} - \pts{\frac{1}{\sqrt{C}n} +\frac{1}{Cn^2}}\pts{n+\frac{n_+}{\Delta_+^2}+\frac{n_-}{\Delta_-^2}}}{d'(n)\pts{1-\frac{2}{\sqrt{C}n}}\pts{1+\frac{2}{\sqrt{C}n}+\frac{4}{Cn^2}} +\tau}\\
            &\geq \frac{\pts{1+\frac{2}{Cn^2}}\pts{\frac{n_+}{\Delta_+}-\frac{n_-}{\Delta_-}} - \frac{2}{\sqrt{C}n}\pts{n+\frac{n_+}{\Delta_+^2}+\frac{n_-}{\Delta_-^2}}}{d +\tau}\\
            &\geq \frac{\pts{1+\frac{2}{Cn^2}}\pts{\frac{n_+}{\Delta_+}-\frac{n_-}{\Delta_-}} - \frac{1}{\tilde{C}}\pts{\frac{n_+}{\Delta_+}+\frac{n_-}{\Delta_-}}}{d +\tau} = \frac{\pts{\frac{C_1+1}{C_1}}\pts{\frac{n_+}{\Delta_+}-\frac{n_-}{\Delta_-}} - \frac{1}{\tilde{C}}\pts{\frac{n_+}{\Delta_+}+\frac{n_-}{\Delta_-}}}{d +\tau}.
        \end{align*}
    $\bullet \; \sididz$ for $i=1,2$,
        \begin{gather*}
            \sididz = \bv_i^\top\bDelta^{-1}\QQtaui\bDelta^{-1}\bv_i.
        \end{gather*}
       Based on Corollary~\ref{cor:wishart_bound_ridge}, let $t=\sqrt{2\log\pts{n/\delta}}$. The variational characterization of eigenvalues gives us, with probability at least $1-\frac{\delta^2}{n^2}$,
        \begin{align*}
            \sididz &\leq \frac{\ndvi^2}{d'(n) - \sqrt{4\log\pts{n/\delta} \cdot d'(n)}+\tau} \leq \pts{\frac{C_1+1}{C_1}}\frac{\frac{n_+}{\Delta_+^2} + \frac{n_-}{\Delta_-^2}}{d+\tau}\\
            \sididz &\geq \frac{\ndvi^2}{d'(n) + \sqrt{4\log\pts{n/\delta} \cdot d'(n)} + 4\log\pts{n/\delta} + \tau} \geq \pts{\frac{C_1-1}{C_1}}\frac{\frac{n_+}{\Delta_+^2} + \frac{n_-}{\Delta_-^2}}{d+\tau}.
        \end{align*}
    $\bullet \; \hijdz = \bd_i^\top\XXztaui\bDelta^{-1}\bv_j = \bd_i^\top\QQtaui\bDelta^{-1}\bv_j$ for $i,j=1,2$. Hence,
        \begin{align*}
            \hijdz &\leq \ndi\ndvj\nnorm[2]{\QQtaui} \leq \frac{C\sqrt{n\pts{\frac{n_+}{\Delta_+^2} + \frac{n_-}{\Delta_-^2}}}\nui}{d'(n) - \sqrt{4\log\pts{n/\delta} \cdot d'(n)} +\tau} \leq\frac{C_1\sqrt{n\nd}\nui}{d +\tau}.\\
            \hijdz &\geq \frac{-C_1\sqrt{n\nd}\nui}{d +\tau}.
        \end{align*}
    $\bullet \; \oididk=\bv_i^\top\bDelta^{-1}\XXktaui\XXk\XXktaui\bDelta^{-1}\bv_i$, for $i=1,2$. We start with the upper bound
        \begin{align}
            \oididk &=\bv_i^\top\bDelta^{-1}\XXktaui\XXk\XXktaui\bDelta^{-1}\bv_i\nonumber\\ 
            &\leq \lbf[1]{\XXk}\nnorm[2]{\XXktaui\bDelta^{-1}\bv_i}^2\nonumber\\
            &\leq \lbf[1]{\XXk}\nnorm[2]{\XXktaui}^2\nnorm[2]{\bDelta^{-1}\bv_i}^2\label{eq:primitive_o_upper},
        \end{align}
    where the first inequality follows by the variational characterization of eigenvalues, and the second inequality applies the sub-multiplicative property of the matrix norm. Next, we focus on analyzing the eigenvalues of $\XXk$ and $\XXk+\tau\I$. We have
        \begin{align*}
            \eqref{eq:primitive_o_upper} &= \lbf[1]\XXk\lbf[1]{\XXktaui}^2\nnorm[2]{\bDelta^{-1}\bv_i}^2\\
            &= \frac{\lbf[1]{\XXktaui}^2}{\lbf[n]{\XXki}}\nnorm[2]{\bDelta^{-1}\bv_i}^2 \leq \frac{\pts{\hc_1/\pts{d+\tau}}^2}{\dc_2/d}\pts{\frac{n_+}{\Delta_+^2} + \frac{n_-}{\Delta_-^2}} = \frac{\hc_3 \nd d}{\pts{d+\tau}^2},
        \end{align*}
    where the inequality follows by applying the primitive bounds of $\suuk$ in Lemma~\ref{lem:primitive_bounds}. Similarly, for the lower bound we have
        \begin{align}
            \oididk &=\bv_i^\top\bDelta^{-1}\XXktaui\XXk\XXktaui\bDelta^{-1}\bv_i\nonumber\\ 
            &\geq \lbf[n]{\XXktaui\XXk\XXktaui}\nnorm[2]{\bDelta^{-1}\bv_i}^2\nonumber\\
            &= \frac{\nnorm[2]{\bDelta^{-1}\bv_i}^2}{\lbf[1]{\pts{\XXk+\tau\I}\XXki\pts{\XXk+\tau\I}}}\nonumber\\
            &\geq \frac{\nnorm[2]{\bDelta^{-1}\bv_i}^2}{\lbf[1]{\XXki}\lbf[1]{\XXk+\tau\I}^2}\label{eq:primitive_o_lower}.
        \end{align}
    Next, we have
         \begin{align*}
            \eqref{eq:primitive_o_lower} &= \frac{\lbf[n]{\XXktaui}^2}{\lbf[1]{\XXki}}\nnorm[2]{\bDelta^{-1}\bv_i}^2 \geq \frac{\pts{\dc_1/\pts{d+\tau}}^2}{\hc_2/d}\pts{\frac{n_+}{\Delta_+^2} + \frac{n_-}{\Delta_-^2}} = \frac{\dc_3 \nd d}{\pts{d+\tau}^2}.
        \end{align*}
    
    Finally, we show the inductive step from $k \to k+1$. The calculation is identical to the one that we did for the proof of Lemma~\ref{lem:primitive_bounds}, since we reuse the $f_{\A_{K+1}}$ function in Lemma~\ref{lem:aux_f_func}, and only the input variables change. (The full proof of the inductive step is provided in Appendix~\ref{app:primitive_recursive_proof} for completeness.)
    This completes the proof of the lemma.
\end{proof}

\section{Proofs of auxiliary lemmas (Lemmas~\ref{lem:aux_compare_a_h},~\ref{lem:aux_f_func},~\ref{lem:aux_det_bound} and~\ref{lem:aux_ineq})} \label{app:aux_lem_proof}
In this section, we provide the proofs of auxiliary technical lemmas used in Appendix~\ref{app:primitive_proof}.
\subsection{Proof of Lemma~\ref{lem:aux_compare_a_h}}

We have
    \begin{align*}
        \frac{1}{2}\pts{\ndpnd}\nuc &\geq \frac{1}{2}\sqrt{\frac{n_+^2}{\Delta_+^2} + \frac{n_-^2}{\Delta_-^2}}\nuc\\
        &= \sqrt{\frac{n_+}{\Delta_+^2}\pts{\frac{n_+\nuc^2}{4}} + \frac{n_-}{\Delta_-^2}\pts{\frac{n_-\nuc^2}{4}}}\\
        &\geq \sqrt{\pts{\frac{n_-\nuc^2}{4}}\pts{\frac{n_+}{\Delta_+^2} + \frac{n_-}{\Delta_-^2}}}\\
        &\geq \sqrt{\pts{\frac{n_-Cn\log(n/\delta)}{4}}\pts{\frac{n_+}{\Delta_+^2} + \frac{n_-}{\Delta_-^2}}}\\
        &= C_1\sqrt{n\pts{\frac{n_+}{\Delta_+^2} + \frac{n_-}{\Delta_-^2}}} = C_1\sqrt{n\nd},
    \end{align*}
    where the second inequality follows by $\Delta_- \leq \Delta_+$, the third inequality follows via the assumptions made in the lemma, and the last equality follows by the choice $C_1=\sqrt{\frac{Cn_-\log(n/\delta)}{4}} > 1$. On the other hand, if $n_- = \alpha n$, we can pick $C_1=\sqrt{\alpha} \frac{\nuc}{2}$.
    This completes the proof of the lemma.
    
\qed

\subsection{Proof of Lemma~\ref{lem:aux_f_func}}
    According to the definition of $f_{\A_k}$, we can compute the function value based on the definition of adjugate matrix of $\A_k$, by a series of algebraic steps. To upper bound $|f_{\A_k}|$, we apply the bounds on the $(k-1)$-th order primitives to get
    \begin{align*}
        \normo{\ffs[\A_{k}]{\begin{matrix}
            x_a, x_b,\\ x_c, x_d
        \end{matrix}}} &= \normo{\pts{\nuk^2-\tkkkmo}x_a x_c + \pts{1+\hkkkmo}\pts{x_a x_d + x_b x_c} -\skkkmo x_b x_d}\\
        &\leq \normo{\nuk^2-\tkkkmo}\normo{x_a} \normo{x_c} + \normo{1+\hkkkmo}\pts{\normo{x_a} \normo{x_d} + \normo{x_b} \normo{x_c}} +\normo{\skkkmo} \normo{x_b} \normo{x_d}\\
        &\leq \nuk^2\normo{x_a} \normo{x_c}  +\normo{\skkkmo} \normo{x_b} \normo{x_d} + \pts{1+\frac{C_1 n\nuk}{d+\tau}}\pts{\normo{x_a} \normo{x_d} + \normo{x_b} \normo{x_c}}\\
        &\leq \nuk^2\normo{x_a} \normo{x_c} + \pts{1+\frac{1}{C_2}}\pts{\normo{x_a} \normo{x_d} + \normo{x_b} \normo{x_c}} +\normo{\skkkmo} \normo{x_b} \normo{x_d},
    \end{align*}
    where we introduce the primitive bounds in the second inequality, and apply Assumption~\ref{asm:dataset_gen}(C) in the last inequality.
\qed
\subsection{Proof of Lemma~\ref{lem:aux_det_bound}}
    According to the definition of $\det(\A_k)$ in Eq.~\eqref{eq:det_def}, we have
    \begin{align*}
        \det(\A_k) = \skkkmo\pts{\nuk^2-\tkkkmo}+\pts{\hkkkmo+1}^2.
    \end{align*}
    Thus, we introduce the upper bound on the $(k-1)$th order primitive to yield
    \begin{align*}
        \det(\A_k) &\leq \pts{\frac{C_1+1}{C_1}}\frac{n}{d+\tau}\nuk^2 + \pts{\frac{C_2 n\nuk}{d+\tau} + 1}^2\\
        &\leq \pts{\frac{C_1+1}{C_1}}\frac{n}{d+\tau}\nuk^2 + \pts{\frac{1}{C_3} + 1}^2\\
        &\leq \pts{\frac{C_4+1}{C_4}}\pts{\frac{n}{d+\tau}\nuk^2 + \frac{C_4+1}{C_4}} \leq \frac{C_5+1}{C_5},
    \end{align*}
    where the second and third inequality follow by applying Assumption~\ref{asm:dataset_gen}(C), i.e.~$\frac{R_+n}{d} \leq \frac{1}{C}$. Similarly, 
    \begin{align*}
        \det(\A_k) &\geq \pts{\frac{C_1-1}{C_1}}^2\frac{n}{d+\tau} \nuk^2 + \pts{\frac{-C_2 n\nuk}{d+\tau} + 1}^2 \geq \pts{\frac{C_3-1}{C_3}}.
    \end{align*}
    This completes the proof of the lemma.
\qed
\subsection{Proof of Lemma~\ref{lem:aux_ineq}}
    Since $\Delta_{\pm} \leq 1$, we have
    \begin{align*}
        \frac{1+\frac{1}{\Delta_\pm^2}}{2} &\leq \frac{1}{\Delta_\pm} + \frac{1}{2}\pts{\frac{1}{\Delta_\pm^2}-1}\\
        \iff 1 + \frac{1}{\Delta_\pm^2} &\leq \frac{1}{2\Delta_\pm}\pts{4 + 2\pts{\frac{1}{\Delta_\pm}-\Delta_\pm}}\\
        \iff n_\pm + \frac{n_\pm}{\Delta_\pm^2} &\leq \frac{n_\pm}{2\Delta_\pm}\pts{4 + 2\pts{\frac{1}{\Delta_\pm}-\Delta_\pm}}.
    \end{align*}
    Then, we have
    \begin{align*}
        n_+ + \frac{n_+}{\Delta_+^2} + n_- + \frac{n_-}{\Delta_-^2} &= n + \frac{n_+}{\Delta_+^2} + \frac{n_-}{\Delta_-^2}\\
        &\leq \frac{n_+}{2\Delta_+}\pts{4 + 2\pts{\frac{1}{\Delta_+}-\Delta_+}} + \frac{n_-}{2\Delta_-}\pts{4 + 2\pts{\frac{1}{\Delta_-}-\Delta_-}}\\
        &\leq \pts{\frac{n_+}{2\Delta_+} + \frac{n_-}{2\Delta_-}} \pts{4 + 2\pts{\frac{1}{\Delta_-}-\Delta_-}},
    \end{align*}
    where the last inequality comes from $\Delta_- \leq \Delta_+$. Hence, it is sufficient to show $\frac{\sqrt{C}n}{\tilde{C}} \geq \pts{4 + 2\pts{\frac{1}{\Delta_-}-\Delta_-}}$. To this end, we have
    \begin{align*}
        \pts{4 + 2\pts{\frac{1}{\Delta_-}-\Delta_-}}^2 &\leq \pts{4 + 2\pts{\frac{1}{\Delta_-}}}^2 = 16 + \frac{4}{\Delta_-^2} + \frac{16}{\Delta_-} \leq 16 + 4n^2 + 16n \leq \frac{Cn^2}{\tilde{C}^2}
    \end{align*}
    for some $\frac{C}{\tilde{C}^2} \geq 36$ since $\frac{1}{n}\leq \Delta_- \leq 1$.
    This completes the proof of the lemma.
\qed
\section{Omitted calculations in proofs of Lemmas~\ref{lem:primitive_bounds} and~\ref{lem:w_primitive_bounds}} \label{app:primitive_recursive_proof}
%
In this section, we provide the detailed calculations of the inductive step, i.e.~going from $k \to k+1$, in the proofs of Lemmas~\ref{lem:primitive_bounds} and~\ref{lem:w_primitive_bounds}.
\subsection{Inductive step for Lemma~\ref{lem:primitive_bounds} (without adjustment weight)}
    $\bullet \; \siikpo$ for $i = 1,2 $: We have
    \begin{align*}
        &\siikpo = \bv_i^\top\XXkpotaui\bv_i\\
        &=\bv_i^\top\XXktaui\bv_i - \bv_i^\top\XXktaui\bL_{k+1}\A_{k+1}^{-1}\bR_{k+1}\XXktaui\bv_i\\
        &= \siik - \mts{\nukpo\sikpok,~\hkpoik,~\sikpok} \times\frac{\text{adj}(\A_{k+1})}{\det(\A_{k+1})}\begin{bmatrix}
        \nukpo\sikpok \\
        \sikpok \\
        \hkpoik
        \end{bmatrix}\\
        &= \frac{\siik \det(\A_{k+1}) - \ffs[\A_{k+1}]{\sikpok, \hkpoik, \sikpok, \hkpoik}}{\det(\A_{k+1})}.
    \end{align*}
    Thus, we can apply Lemma~\ref{lem:aux_f_func} to get 
    \begin{align*}
        \siikpo &\leq \frac{\normo{\siik \det(\A_{k+1})} + \normo{\ffs[\A_{k+1}]{\sikpok, \hkpoik, \sikpok, \hkpoik}}}{\det(\A_{k+1})}\\
        &\leq \frac{1}{\det(\A_{k+1})}\Bigg(\det(\A_{k+1})\normo{\siik} + \nukpo^2\normo{\sikpok}\normo{\sikpok}\\
        &+ \pts{1+\frac{1}{C_1}}\pts{\normo{\sikpok}\normo{\hkpoik} + \normo{\hkpoik}\normo{\sikpok}} + \normo{\skpokpok}\normo{\hkpoik}\normo{\hkpoik}\Bigg)\\
        &\leq \frac{1}{\det(\A_{k+1})} \pts{\frac{C_2+1}{C_2}\frac{n}{d+\tau} + \frac{C_3+1}{C_3}\frac{\nukpo^2n^2}{\pts{d+\tau}^2} + \frac{C_4+1}{C_4}\frac{C_5\nukpo n^2}{\pts{d+\tau}^2} + \frac{C_6\nukpo^2n^3}{\pts{d+\tau}^3}}\\
        & \leq \frac{C_7+1}{C_7}\frac{n}{d+\tau},
    \end{align*}
    where in the third inequality we apply the lower bound on $\det(\A_{k+1})$ (Lemma~\ref{lem:aux_det_bound}) and the $k$-th order primitive bounds, and in the last inequality, we apply Assumption~\ref{asm:dataset_gen}(C), i.e.~$d \geq CR_+n$.  Similarly, 
    \begin{align*}
        \siikpo &\geq \frac{\siik \det(\A_{k+1}) - \normo{\ffs[\A_{k+1}]{\sikpok, \hkpoik, \sikpok, \hkpoik}}}{\det(\A_{k+1})}\\
        &\geq \frac{1}{\det(\A_{k+1})}\Bigg(\det(\A_{k+1})\siik- \Bigg(\nukpo^2\normo{\sikpok}\normo{\sikpok}\\
        &+ \pts{1+\frac{1}{C_1}}\pts{\normo{\sikpok}\normo{\hkpoik} + \normo{\hkpoik}\normo{\sikpok}} + \normo{\skpokpok}\normo{\hkpoik}\normo{\hkpoik}\Bigg)\Bigg)\\
        &\geq \frac{1}{\det(\A_{k+1})} \pts{\frac{C_2-1}{C_2}\frac{n}{d+\tau}-\frac{C_3+1}{C_3}\frac{\nukpo^2n^2}{\pts{d+\tau}^2} - \frac{C_4+1}{C_4}\frac{C_5\nukpo n^2}{\pts{d+\tau}^2} - \frac{C_6\nukpo^2n^3}{\pts{d+\tau}^3} }\\
        &\geq \frac{C_7-1}{C_7}\frac{n}{d+\tau}.
    \end{align*}
    $\bullet \; \sotkpo$: We have
    \begin{align*}
        \sotkpo &= \bv_1^\top\XXkpotaui\bv_2\\
        &=\bv_1^\top\XXktaui\bv_2 - \bv_1^\top\XXktaui\bL_{k+1}\A_{k+1}^{-1}\bR_{k+1}\XXktaui\bv_2\\
        &= \sotk - \mts{\nukpo\sokpok,~\hkpook,~\sokpok} \times\frac{\text{adj}(\A_{k+1})}{\det(\A_{k+1})}\begin{bmatrix}
        \nukpo\stkpok \\
        \stkpok \\
        \hkpotk
        \end{bmatrix}\\
        &= \frac{\sotk \det(\A_{k+1}) - \ffs[\A_{k+1}]{\sokpok, \hkpook, \stkpok, \hkpotk}}{\det(\A_{k+1})}.
    \end{align*}
    Thus, we can apply Lemma~\ref{lem:aux_f_func} to get 
    \begin{align*}
        \sotkpo &\leq \frac{\normo{\sotk \det(\A_{k+1})} + \normo{\ffs[\A_{k+1}]{\sokpok, \hkpook, \stkpok, \hkpotk}}}{\det(\A_{k+1})}\\
        &\leq \frac{1}{\det(\A_{k+1})}\Bigg(\det(\A_{k+1})\normo{\sotk} + \nukpo^2\normo{\sokpok}\normo{\stkpok} + \normo{\skpokpok}\normo{\hkpook}\normo{\hkpotk}\\
        &\myquad[2] + \pts{1+\frac{1}{C_1}}\pts{\normo{\sokpok}\normo{\hkpotk} + \normo{\hkpook}\normo{\stkpok}}\Bigg)\\
        & \leq \frac{1}{\det(\A_{k+1})} \Bigg(\frac{C_2+1}{C_2}\pts{\frac{C_3+1}{C_3}\frac{n_+}{d+\tau}-\frac{C_3-1}{C_3}\frac{n_-}{d+\tau}}\\
        & + \frac{C_4+1}{C_4}\pts{\frac{C_5+1}{C_5}\frac{n_+}{d+\tau}-\frac{C_5-1}{C_5}\frac{n_-}{d+\tau}}\frac{\nukpo^2 n}{d+\tau} + \frac{C_6\nukpo^2n^3}{\pts{d+\tau}^3}\\
        &+ \frac{C_7+1}{C_7}\Bigg(\pts{\frac{C_8+1}{C_8}\frac{n_+}{d+\tau}-\frac{C_8-1}{C_8}\frac{n_-}{d+\tau}}\frac{C_9\nukpo n}{d+\tau} + \frac{C_{10}\nukpo n^2}{\pts{d+\tau}^2}\Bigg)\Bigg)\\
        & \leq \frac{C_{11}+1}{C_{11}}\frac{n_+}{d+\tau}-\frac{C_{11}-1}{C_{11}}\frac{n_-}{d+\tau},
    \end{align*}
    where in the third inequality we apply the lower bound on $\det(\A_{k+1})$ (Lemma~\ref{lem:aux_det_bound}) and the $k$-th order primitive bounds, and in the last inequality, we apply Assumption~\ref{asm:dataset_gen}(C), i.e.~$d \geq CR_+n$.  Similarly,
    \begin{align*}
        \sotkpo &\geq \frac{\sotk \det(\A_{k+1}) - \normo{\ffs[\A_{k+1}]{\sokpok, \hkpook, \stkpok, \hkpotk}}}{\det(\A_{k+1})}\\
        &\geq \frac{1}{\det(\A_{k+1})}\Bigg(\det(\A_{k+1})\sotk -\Bigg(\nukpo^2\normo{\sokpok}\normo{\stkpok} + \normo{\skpokpok}\normo{\hkpook}\normo{\hkpotk}\\
        &\myquad[2]+ \pts{1+\frac{1}{C_1}}\pts{\normo{\sokpok}\normo{\hkpotk} + \normo{\hkpook}\normo{\stkpok}}\Bigg)\Bigg)\\
        & \geq \frac{1}{\det(\A_{k+1})} \Bigg(\frac{C_2-1}{C_2}\pts{\frac{C_3-1}{C_3}\frac{n_+}{d+\tau}-\frac{C_3+1}{C_3}\frac{n_-}{d+\tau}}\\
        & - \frac{C_4+1}{C_4}\pts{\frac{C_5+1}{C_5}\frac{n_+}{d+\tau}-\frac{C_5-1}{C_5}\frac{n_-}{d+\tau}}\frac{\nukpo^2 n}{d+\tau}- \frac{C_6\nukpo^2n^3}{\pts{d+\tau}^3}\\
        & - \frac{C_7+1}{C_7}\Bigg(\pts{\frac{C_8+1}{C_8}\frac{n_+}{d+\tau}-\frac{C_8-1}{C_8}\frac{n_-}{d+\tau}}\frac{C_9\nukpo n}{d+\tau} + \frac{C_{10}\nukpo n^2}{\pts{d+\tau}^2}\Bigg)\Bigg)\\
        & \geq \frac{C_{11}-1}{C_{11}}\frac{n_+}{d+\tau}-\frac{C_{11}+1}{C_{11}}\frac{n_-}{d+\tau}.
    \end{align*}
    $\bullet \; \tiikpo$ for $i = 1,2 $: We have
    \begin{align*}
        \tiikpo &= \bd_i^\top\XXkpotaui\bd_i\\
        &=\bd_i^\top\XXktaui\bd_i - \bd_i^\top\XXktaui\bL_{k+1}\A_{k+1}^{-1}\bR_{k+1}\XXktaui\bd_i\\
        &= \tiik - \mts{\nukpo\hikpok,~\tikpok,~\hikpok} \times\frac{\text{adj}(\A_{k+1})}{\det(\A_{k+1})}\begin{bmatrix}
        \nukpo\hikpok \\
        \hikpok \\
        \tikpok
        \end{bmatrix}\\
        &= \frac{\tiik \det(\A_{k+1}) - \ffs[\A_{k+1}]{\hikpok, \tikpok, \hikpok, \tikpok}}{\det(\A_{k+1})}.
    \end{align*}
    Thus, we can apply Lemma~\ref{lem:aux_f_func} to get 
    \begin{align*}
        \tiikpo &\leq \frac{\normo{\tiik \det(\A_{k+1})} + \normo{\ffs[\A_{k+1}]{\hikpok, \tikpok, \hikpok, \tikpok}}}{\det(\A_{k+1})}\\ 
        &\leq \frac{1}{\det(\A_{k+1})}\Bigg(\det(\A_{k+1})\normo{\tiik} + \nukpo^2\normo{\hikpok}\normo{\hikpok} + \normo{\skpokpok}\normo{\tikpok}\normo{\tikpok}\\
        &\myquad[2] + \pts{1+\frac{1}{C_1}}\pts{\normo{\hikpok}\normo{\tikpok} + \normo{\tikpok}\normo{\hikpok}}\Bigg)\\
        & \leq \frac{1}{\det(\A_{k+1})} \Bigg(\frac{C_2+1}{C_2}\frac{C_3n\nui^2}{d+\tau} + \frac{C_4\nukpo^2\nui^2n^2}{\pts{d+\tau}^2}\\
        &\myquad[2] + \frac{C_5\nui^2\nukpo^2n^3}{\pts{d+\tau}^3} + \frac{C_6\nui^2\nukpo n^2}{\pts{d+\tau}^2}\Bigg)\\
        & \leq \frac{C_7n\nui^2}{d+\tau},
    \end{align*}
    where in the third inequality we apply the lower bound on $\det(\A_{k+1})$ (Lemma~\ref{lem:aux_det_bound}) and the $k$-th order primitive bounds, and in the last inequality, we apply Assumption~\ref{asm:dataset_gen}(C), i.e.~$d \geq CR_+n$.\\
    $\bullet \; \totkpo$: We have
    \begin{align*}
        \totkpo &= \bd_1^\top\XXkpotaui\bd_2\\
        &=\bd_1^\top\XXktaui\bd_2 - \bd_1^\top\XXktaui\bL_{k+1}\A_{k+1}^{-1}\bR_{k+1}\XXktaui\bd_2\\
        &= \totk - \mts{\nukpo\hokpok,~\tokpok,~\hokpok} \times\frac{\text{adj}(\A_{k+1})}{\det(\A_{k+1})}\begin{bmatrix}
        \nukpo\htkpok \\
        \htkpok \\
        \ttkpok
        \end{bmatrix}\\
        &= \frac{\totk \det(\A_{k+1}) - \ffs[\A_{k+1}]{\hokpok, \tokpok, \htkpok, \ttkpok}}{\det(\A_{k+1})}.
    \end{align*}
    Thus, we can apply Lemma~\ref{lem:aux_f_func} to get 
    \begin{align*}
        \totkpo &\leq \frac{\normo{\totk \det(\A_{k+1})} + \normo{\ffs[\A_{k+1}]{\hokpok, \tokpok, \htkpok, \ttkpok}}}{\det(\A_{k+1})}\\
        &\leq \frac{1}{\det(\A_{k+1})}\Bigg(\det(\A_{k+1})\normo{\totk} + \nukpo^2\normo{\hokpok}\normo{\htkpok} + \normo{\skpokpok}\normo{\tokpok}\normo{\ttkpok}\\
        &\myquad[2] + \pts{1+\frac{1}{C_1}}\pts{\normo{\hokpok}\normo{\ttkpok} + \normo{\tokpok}\normo{\htkpok}}\Bigg)\\
        & \leq \frac{1}{\det(\A_{k+1})} \Bigg(\frac{C_2+1}{C_2}\frac{C_3n\nuo\nut}{d+\tau} + \frac{C_4\nukpo^2\nuo\nut n^2}{\pts{d+\tau}^2}\\
        &\myquad[2] + \frac{C_5\nuo\nut\nukpo^2n^3}{\pts{d+\tau}^3} + \frac{C_6\nuo\nut\nukpo n^2}{\pts{d+\tau}^2}\Bigg)\\
        & \leq \frac{C_7n\nuo\nut}{d+\tau},
    \end{align*}
    where in the third inequality we apply the lower bound on $\det(\A_{k+1})$ (Lemma~\ref{lem:aux_det_bound}) and the $k$-th order primitive bounds, and in the last inequality, we apply Assumption~\ref{asm:dataset_gen}(C), i.e.~$d \geq CR_+n$. Similarly, 
    \begin{align*}
        \totkpo &\geq \frac{\totk \det(\A_{k+1}) - \normo{\ffs[\A_{k+1}]{\hokpok, \tokpok, \htkpok, \ttkpok}}}{\det(\A_{k+1})}\\ 
        &\geq \frac{1}{\det(\A_{k+1})}\Bigg(\det(\A_{k+1})\totk -\Bigg( \nukpo^2\normo{\hokpok}\normo{\htkpok} + \normo{\skpokpok}\normo{\tokpok}\normo{\ttkpok}\\
        &\myquad[2] + \pts{1+\frac{1}{C_1}}\pts{\normo{\hokpok}\normo{\ttkpok} + \normo{\tokpok}\normo{\htkpok}}\Bigg)\Bigg)\\
        & \geq \frac{1}{\det(\A_{k+1})} \Bigg(-\frac{C_2+1}{C_2}\frac{C_3n\nuo\nut}{d+\tau} - \frac{C_4\nukpo^2\nuo\nut n^2}{\pts{d+\tau}^2}\\
        &\myquad[2] - \frac{C_5\nuo\nut\nukpo^2n^3}{\pts{d+\tau}^3} - \frac{C_6\nuo\nut\nukpo n^2}{\pts{d+\tau}^2}\Bigg)\\
        & \geq -\frac{C_7n\nuo\nut}{d+\tau}.
    \end{align*}
    $\bullet \; \hijkpo = \bd_i^\top\XXkpotaui\bv_j$, for $i,j=1,2$: We have
    \begin{align*}
        \hijkpo &= \bd_i^\top\XXkpotaui\bv_j\\
        &=\bd_i^\top\XXktaui\bv_j - \bd_i^\top\XXktaui\bL_{k+1}\A_{k+1}^{-1}\bR_{k+1}\XXktaui\bv_j\\
        &= \hijk - \mts{\nukpo\hikpok,~\tikpok,~\hikpok} \times\frac{\text{adj}(\A_{k+1})}{\det(\A_{k+1})}\begin{bmatrix}
        \nukpo\sjkpok \\
        \sjkpok \\
        \hkpojk
        \end{bmatrix}\\
        &= \frac{\hijk \det(\A_{k+1}) - \ffs[\A_{k+1}]{\hikpok, \tikpok, \sjkpok, \hkpojk}}{\det(\A_{k+1})}.
    \end{align*}
    Thus, we can apply Lemma~\ref{lem:aux_f_func} to get 
    \begin{align*}
        \hijkpo &\leq \frac{\normo{\hijk \det(\A_{k+1})} + \normo{\ffs[\A_{k+1}]{\hikpok, \tikpok, \sjkpok, \hkpojk}}}{\det(\A_{k+1})}\\ 
        &\leq \frac{1}{\det(\A_{k+1})}\Bigg(\det(\A_{k+1})\normo{\hijk} + \nukpo^2\normo{\hikpok}\normo{\sjkpok} + \normo{\skpokpok}\normo{\tikpok}\normo{\hkpojk}\\
        &\myquad[2] + \pts{1+\frac{1}{C_1}}\pts{\normo{\hikpok}\normo{\hkpojk} + \normo{\tikpok}\normo{\sjkpok}}\Bigg)\\
        & \leq \frac{1}{\det(\A_{k+1})} \Bigg(\frac{C_2+1}{C_2}\frac{C_3n\nui}{d+\tau} + \frac{C_4\nukpo^2\nui n^2}{\pts{d+\tau}^2}\\
        &\myquad[2] + \frac{C_5\nui\nukpo^2n^3}{\pts{d+\tau}^3} + \frac{C_6\nui\nukpo n^2}{\pts{d+\tau}^2}\Bigg)\\
        & \leq \frac{C_7n\nui}{d+\tau},
    \end{align*}
    where in the third inequality we apply the lower bound on $\det(\A_{k+1})$  (Lemma~\ref{lem:aux_det_bound}) and the $k$-th order primitive bounds, and in the last inequality, we apply Assumption~\ref{asm:dataset_gen}(C), i.e.~$d \geq CR_+n$. Similarly, 
    \begin{align*}
        \hijkpo &\geq \frac{\hijk \det(\A_{k+1}) - \normo{\ffs[\A_{k+1}]{\hikpok, \tikpok, \sjkpok, \hkpojk}}}{\det(\A_{k+1})}\\ 
        &\geq \frac{1}{\det(\A_{k+1})}\Bigg(\det(\A_{k+1})\hijk -\Bigg( \nukpo^2\normo{\hikpok}\normo{\sjkpok} + \normo{\skpokpok}\normo{\tikpok}\normo{\hkpojk}\\
        &\myquad[2] + \pts{1+\frac{1}{C_1}}\pts{\normo{\hikpok}\normo{\hkpojk} + \normo{\tikpok}\normo{\sjkpok}}\Bigg)\Bigg)\\
        & \geq \frac{1}{\det(\A_{k+1})} \Bigg(-\frac{C_2+1}{C_2}\frac{C_3n\nui}{d+\tau} - \frac{C_4\nukpo^2\nui n^2}{\pts{d+\tau}^2}\\
        & \myquad[2]- \frac{C_5\nui\nukpo^2n^3}{\pts{d+\tau}^3} - \frac{C_6\nui\nukpo n^2}{\pts{d+\tau}^2}\Bigg)\\
        & \geq -\frac{C_7n\nui}{d+\tau}.
    \end{align*}
    $\bullet \; \suukpo$: We have
    \begin{align*}
        &\suukpo = \bu^\top\XXkpotaui\bu\\
        &=\bu^\top\XXktaui\bu - \bu^\top\XXktaui\bL_{k+1}\A_{k+1}^{-1}\bR_{k+1}\XXktaui\bu\\
        &= \suuk - \mts{\nukpo\sukpok,~\hkpouk,~\sukpok} \times\frac{\text{adj}(\A_{k+1})}{\det(\A_{k+1})}\begin{bmatrix}
        \nukpo\sukpok \\
        \sukpok \\
        \hkpouk
        \end{bmatrix}\\
        &= \frac{\suuk \det(\A_{k+1}) - \ffs[\A_{k+1}]{\sukpok, \hkpouk, \sukpok, \hkpouk}}{\det(\A_{k+1})}.
    \end{align*}
    Thus, we can apply Lemma~\ref{lem:aux_f_func} to get 
    \begin{align*}
        \suukpo &\leq \frac{\normo{\suuk \det(\A_{k+1})} + \normo{\ffs[\A_{k+1}]{\sukpok, \hkpouk, \sukpok, \hkpouk}}}{\det(\A_{k+1})}\\
        &\leq \frac{1}{\det(\A_{k+1})}\Bigg(\det(\A_{k+1})\normo{\suuk} + \nukpo^2\normo{\sukpok}\normo{\sukpok}\\
        &+ \pts{1+\frac{1}{C_1}}\pts{\normo{\sukpok}\normo{\hkpouk} + \normo{\hkpouk}\normo{\sukpok}} + \normo{\skpokpok}\normo{\hkpouk}\normo{\hkpouk}\Bigg)\\
        &\leq \frac{1}{\det(\A_{k+1})} \pts{\frac{C_2+1}{C_2}\frac{1}{d+\tau} + \frac{C_3\nukpo^2n}{\pts{d+\tau}^2} + \frac{C_4\nukpo n}{\pts{d+\tau}^2} + \frac{C_5\nukpo^2n^2}{\pts{d+\tau}^3}}\\
        & \leq \frac{1}{\det(\A_{k+1})} \pts{\frac{C_6+1}{C_6}\frac{1}{d+\tau}} \leq \frac{C_7+1}{C_7}\frac{1}{d+\tau},
    \end{align*}
    where in the third inequality we apply the lower bound on $\det(\A_{k+1})$  (Lemma~\ref{lem:aux_det_bound}) and the $k$-th order primitive bounds, and in the last inequality, we apply Assumption~\ref{asm:dataset_gen}(C) such that $d \geq CR_+n$. Similarly, 
    \begin{align*}
        \suukpo &\geq \frac{\suuk \det(\A_{k+1}) - \normo{\ffs[\A_{k+1}]{\sukpok, \hkpouk, \sukpok, \hkpouk}}}{\det(\A_{k+1})}\\ 
        &\geq \frac{1}{\det(\A_{k+1})}\Bigg(\det(\A_{k+1})\suuk- \Bigg(\nukpo^2\normo{\sukpok}\normo{\sukpok}\\
        &+ \pts{1+\frac{1}{C_1}}\pts{\normo{\sukpok}\normo{\hkpouk} + \normo{\hkpouk}\normo{\sukpok}} + \normo{\skpokpok}\normo{\hkpouk}\normo{\hkpouk}\Bigg)\Bigg)\\
        &\geq \frac{1}{\det(\A_{k+1})} \pts{\frac{C_2-1}{C_2}\frac{1}{d+\tau} - \frac{C_3\nukpo^2n}{\pts{d+\tau}^2} - \frac{C_4\nukpo n}{\pts{d+\tau}^2} - \frac{C_5\nukpo^2n^2}{\pts{d+\tau}^3}}\\
        & \geq \frac{1}{\det(\A_{k+1})} \pts{\frac{C_6-1}{C_6}\frac{1}{d+\tau}} = \frac{C_7-1}{C_7}\frac{1}{d+\tau}.
    \end{align*}
    $\bullet \; \suikpo$ for $i=1,2$: We have
    \begin{align*}
        \suikpo &= \bu^\top\XXkpotaui\bv_i\\
        &=\bu^\top\XXktaui\bv_i - \bu^\top\XXktaui\bL_{k+1}\A_{k+1}^{-1}\bR_{k+1}\XXktaui\bv_i\\
        &= \suik - \mts{\nukpo\sukpok,~\hkpouk,~\sukpok} \times\frac{\text{adj}(\A_{k+1})}{\det(\A_{k+1})}\begin{bmatrix}
        \nukpo\sikpok \\
        \sikpok \\
        \hkpoik
        \end{bmatrix}\\
        &= \frac{\suik \det(\A_{k+1}) - \ffs[\A_{k+1}]{\sukpok, \hkpouk, \sikpok, \hkpoik}}{\det(\A_{k+1})}.
    \end{align*}
    Thus, we can apply Lemma~\ref{lem:aux_f_func} to get 
    \begin{align*}
        \suikpo &\leq \frac{\normo{\suik \det(\A_{k+1})} + \normo{\ffs[\A_{k+1}]{\sukpok, \hkpouk, \sikpok, \hkpoik}}}{\det(\A_{k+1})}\\
        &\leq \frac{1}{\det(\A_{k+1})}\Bigg(\det(\A_{k+1})\normo{\suik} + \nukpo^2\normo{\sukpok}\normo{\sikpok} + \normo{\skpokpok}\normo{\hkpouk}\normo{\hkpoik}\\
        &\myquad[2] + \pts{1+\frac{1}{C_1}}\pts{\normo{\sukpok}\normo{\hkpoik} + \normo{\hkpouk}\normo{\sikpok}}\Bigg)\\
        & \leq \frac{1}{\det(\A_{k+1})} \Bigg(\frac{C_2+1}{C_2}\frac{C_3\sqrt{n}}{d+\tau} + \frac{C_4+1}{C_4}\frac{C_5\sqrt{n}}{d+\tau}\frac{\nukpo^2 n}{d+\tau}+ \frac{C_6\nukpo^2n^2\sqrt{n}}{\pts{d+\tau}^3}\\
        &+ \frac{C_7+1}{C_7}\Bigg(\frac{C_8\sqrt{n}}{d+\tau}\frac{\nukpo n}{d+\tau} + \frac{C_9\nukpo n\sqrt{n}}{\pts{d+\tau}^2}\Bigg)\Bigg)\\
        & \leq \frac{C_{10}\sqrt{n}}{d+\tau},
    \end{align*}
    where in the third inequality we apply the lower bound on $\det(\A_{k+1})$ (Lemma~\ref{lem:aux_det_bound}) and the $k$-th order primitive bounds, and in the last inequality, we apply Assumption~\ref{asm:dataset_gen}(C) such that $d \geq CR_+n$.  Similarly,
    \begin{align*}
        \suikpo &\geq \frac{\suik \det(\A_{k+1}) - \normo{\ffs[\A_{k+1}]{\sukpok, \hkpouk, \sikpok, \hkpoik}}}{\det(\A_{k+1})}\\ 
        &\geq \frac{1}{\det(\A_{k+1})}\Bigg(\det(\A_{k+1})\suik -\Bigg(\nukpo^2\normo{\sukpok}\normo{\sikpok} + \normo{\skpokpok}\normo{\hkpouk}\normo{\hkpoik}\\
        &\myquad[2]+ \pts{1+\frac{1}{C_1}}\pts{\normo{\sukpok}\normo{\hkpoik} + \normo{\hkpouk}\normo{\sikpok}}\Bigg)\Bigg)\\
        & \geq \frac{1}{\det(\A_{k+1})} \Bigg(-\frac{C_2-1}{C_2}\frac{C_3\sqrt{n}}{d+\tau} - \frac{C_4+1}{C_4}\frac{C_5\sqrt{n}}{d+\tau}\frac{\nukpo^2 n}{d+\tau} - \frac{C_6\nukpo^2n^2\sqrt{n}}{\pts{d+\tau}^3}\\
        &\myquad[1]  - \frac{C_7+1}{C_7}\Bigg(\frac{C_8\sqrt{n}}{d+\tau}\frac{\nukpo n}{d+\tau} + \frac{C_9\nukpo n\sqrt{n}}{\pts{d+\tau}^2}\Bigg)\Bigg)\\
        & \geq -\frac{C_{10}\sqrt{n}}{d+\tau}.
    \end{align*}
    $\bullet \; \hiukpo = \bd_i^\top\XXkpotaui\bu$, for $i=1,2$: We have
    \begin{align*}
        \hiukpo &= \bd_i^\top\XXkpotaui\bu\\
        &=\bd_i^\top\XXktaui\bu - \bd_i^\top\XXktaui\bL_{k+1}\A_{k+1}^{-1}\bR_{k+1}\XXktaui\bu\\
        &= \hiuk - \mts{\nukpo\hikpok,~\tikpok,~\hikpok} \times\frac{\text{adj}(\A_{k+1})}{\det(\A_{k+1})}\begin{bmatrix}
        \nukpo\sukpok \\
        \sukpok \\
        \hkpouk
        \end{bmatrix}\\
        &= \frac{\hiuk \det(\A_{k+1}) - \ffs[\A_{k+1}]{\hikpok, \tikpok, \sukpok, \hkpouk}}{\det(\A_{k+1})}.
    \end{align*}
    Thus, we can apply Lemma~\ref{lem:aux_f_func} to get 
    \begin{align*}
        \hiukpo &\leq \frac{\normo{\hiuk \det(\A_{k+1})} + \normo{\ffs[\A_{k+1}]{\hikpok, \tikpok, \sukpok, \hkpouk}}}{\det(\A_{k+1})}\\
        &\leq \frac{1}{\det(\A_{k+1})}\Bigg(\det(\A_{k+1})\normo{\hiuk} + \nukpo^2\normo{\hikpok}\normo{\sukpok} + \normo{\skpokpok}\normo{\tikpok}\normo{\hkpouk}\\
        &\myquad[2] + \pts{1+\frac{1}{C_1}}\pts{\normo{\hikpok}\normo{\hkpouk} + \normo{\tikpok}\normo{\sukpok}}\Bigg)\\
        & \leq \frac{1}{\det(\A_{k+1})} \Bigg(\frac{C_2+1}{C_2}\frac{C_3\sqrt{n}\nui}{d+\tau} + \frac{C_4\nukpo^2\nui n\sqrt{n}}{\pts{d+\tau}^2}\\
        &\myquad[2] + \frac{C_5\nui\nukpo^2n^2\sqrt{n}}{\pts{d+\tau}^3} + \frac{C_6\nui\nukpo n\sqrt{n}}{\pts{d+\tau}^2}\Bigg)\\
        & \leq \frac{C_7\sqrt{n}\nui}{d+\tau},
    \end{align*}
    where in the third inequality we apply the lower bound on $\det(\A_{k+1})$ (Lemma~\ref{lem:aux_det_bound}) and the $k$-th order primitive bounds, and in the last inequality, we apply Assumption~\ref{asm:dataset_gen}(C) such that $d \geq CR_+n$. Similarly, 
    \begin{align*}
        \hiukpo &\geq \frac{\hiuk \det(\A_{k+1}) - \normo{\ffs[\A_{k+1}]{\hikpok, \tikpok, \sukpok, \hkpouk}}}{\det(\A_{k+1})}\\
        &\geq \frac{1}{\det(\A_{k+1})}\Bigg(\det(\A_{k+1})\hiuk -\Bigg( \nukpo^2\normo{\hikpok}\normo{\sukpok} + \normo{\skpokpok}\normo{\tikpok}\normo{\hkpouk}\\
        &\myquad[2] + \pts{1+\frac{1}{C_1}}\pts{\normo{\hikpok}\normo{\hkpouk} + \normo{\tikpok}\normo{\sukpok}}\Bigg)\Bigg)\\
        & \geq \frac{1}{\det(\A_{k+1})} \Bigg(-\frac{C_2+1}{C_2}\frac{C_3\sqrt{n}\nui}{d+\tau} - \frac{C_4\nukpo^2\nui n\sqrt{n}}{\pts{d+\tau}^2}\\
        & \myquad[2]- \frac{C_5\nui\nukpo^2n^2\sqrt{n}}{\pts{d+\tau}^3} - \frac{C_6\nui\nukpo n\sqrt{n}}{\pts{d+\tau}^2}\Bigg)\\
        & \geq -\frac{C_7\sqrt{n}\nui}{d+\tau}.
    \end{align*}

%
\subsection{Inductive step for Lemma~\ref{lem:w_primitive_bounds} (with adjustment weight)}
    $\bullet \; \sidikpo$ for $i = 1,2 $: We have
    \begin{align*}
        \sidikpo &= \bv_i^\top\bDelta^{-1}\XXkpotaui\bv_i\\
        &=\bv_i^\top\bDelta^{-1}\XXktaui\bv_i - \bv_i^\top\bDelta^{-1}\XXktaui\bL_{k+1}\A_{k+1}^{-1}\bR_{k+1}\XXktaui\bv_i\\
        &= \sidik - \mts{\nukpo\sidkpok,~\hkpoidk,~\sidkpok} \times\frac{\text{adj}(\A_{k+1})}{\det(\A_{k+1})}\begin{bmatrix}
        \nukpo\sikpok \\
        \sikpok \\
        \hkpoik
        \end{bmatrix}\\
        &= \frac{\sidik \det(\A_{k+1}) - \ffs[\A_{k+1}]{\sidkpok, \hkpoidk, \sikpok, \hkpoik}}{\det(\A_{k+1})}.
    \end{align*}
    Thus, we can apply Lemma~\ref{lem:aux_f_func} to get 
    \begin{align*}
        \sidikpo &\leq \frac{\normo{\sidik \det(\A_{k+1})} + \normo{\ffs[\A_{k+1}]{\sidkpok, \hkpoidk, \sikpok, \hkpoik}}}{\det(\A_{k+1})}\\
        &\leq \frac{1}{\det(\A_{k+1})}\Bigg(\det(\A_{k+1})\normo{\sidik} + \nukpo^2\normo{\sidkpok}\normo{\sikpok} + \normo{\skpokpok}\normo{\hkpoidk}\normo{\hkpoik}\\
        &\myquad[2] + \pts{1+\frac{1}{C_1}}\pts{\normo{\sidkpok}\normo{\hkpoik} + \normo{\hkpoidk}\normo{\sikpok}}\Bigg)\\
        & \leq \frac{1}{\det(\A_{k+1})} \Bigg(\frac{C_2+1}{C_2}\frac{\pts{\hc_3+\frac{1}{C_4}}\pts{\ndpnd}}{d+\tau}\\
        &\myquad[2]+ \frac{C_5+1}{C_5}\frac{\pts{\hc_6+\frac{1}{C_7}}\pts{\ndpnd}}{d+\tau}\frac{\nukpo^2n}{d+\tau}  + \frac{C_8+1}{C_8}\frac{C_9\nukpo^2n^2\sqrt{n\nd}}{\pts{d+\tau}^3}\\
        &\myquad[2] + \frac{C_{10}+1}{C_{10}}\Bigg(\frac{\pts{\hc_{11}+\frac{1}{C_{12}}}\pts{\ndpnd}}{d+\tau}\frac{C_{13}n\nukpo}{d+\tau} + \frac{C_{14}n\sqrt{n\nd}\nukpo}{\pts{d+\tau}^2}\Bigg)\Bigg),
    \end{align*}
    where in the third inequality we apply the lower bound on $\det(\A_{k+1})$ (Lemma~\ref{lem:aux_det_bound}) and the $k$-th order primitive bounds. Next, Lemma~\ref{lem:aux_compare_a_h} tells us that $\sqrt{n\nd}\leq \frac{1}{C}\pts{\ndpnd}\nut$. We then have
    \begin{align*}
        \sidikpo &\leq \frac{1}{\det(\A_{k+1})} \Bigg(\frac{C_2+1}{C_2}\frac{\pts{\hc_3+\frac{1}{C_4}}\pts{\ndpnd}}{d+\tau}\\
        &\myquad[2] + \frac{C_5+1}{C_5}\frac{\pts{\hc_6+\frac{1}{C_7}}\pts{\ndpnd}}{d+\tau}\frac{\nukpo^2n}{d+\tau} + \frac{C_8+1}{C_8}\frac{C_9\nukpo^2\nut n^2 \frac{1}{C}\pts{\ndpnd}}{\pts{d+\tau}^3}\\
        &\myquad[2] + \frac{C_{10}+1}{C_{10}}\Bigg(\frac{\pts{\hc_{11}+\frac{1}{C_{12}}}\pts{\ndpnd}}{d+\tau}\frac{C_{13}n\nukpo}{d+\tau} + \frac{C_{14}n\frac{1}{C}\pts{\ndpnd}\nut\nukpo}{\pts{d+\tau}^2}\Bigg)\Bigg)\\
        & \leq \frac{\pts{\hc_{15}+\frac{1}{C_{16}}}\pts{\ndpnd}}{d+\tau},
    \end{align*} 
    where in the last inequality, we apply Assumption~\ref{asm:dataset_gen}(C), i.e.~$d \geq CR_+n$.  Similarly,
    \begin{align*}
        \sidikpo &\geq \frac{\sidik \det(\A_{k+1}) - \normo{\ffs[\A_{k+1}]{\sidkpok, \hkpoidk, \sikpok, \hkpoik}}}{\det(\A_{k+1})}\\
        &\geq \frac{1}{\det(\A_{k+1})}\Bigg(\det(\A_{k+1})\sidik -\Bigg(\nukpo^2\normo{\sidkpok}\normo{\sikpok} + \normo{\skpokpok}\normo{\hkpoidk}\normo{\hkpoik}\\
        &\myquad[2] + \pts{1+\frac{1}{C_1}}\pts{\normo{\sidkpok}\normo{\hkpoik} + \normo{\hkpoidk}\normo{\sikpok}}\Bigg)\Bigg)\\
        & \geq \frac{1}{\det(\A_{k+1})} \Bigg(\frac{C_2-1}{C_2}\frac{\pts{\hc_3-\frac{1}{C_4}}\pts{\ndpnd}}{d+\tau}\\
        &\myquad[2] - \frac{C_5+1}{C_5}\frac{\pts{\hc_6+\frac{1}{C_7}}\pts{\ndpnd}}{d+\tau}\frac{\nukpo^2n}{d+\tau} - \frac{C_8+1}{C_8}\frac{C_9\nukpo^2\nut n^2 \frac{1}{C}\pts{\ndpnd}}{\pts{d+\tau}^3}\\
        &\myquad[2] - \frac{C_{10}+1}{C_{10}}\Bigg(\frac{\pts{\hc_{11}+\frac{1}{C_{12}}}\pts{\ndpnd}}{d+\tau}\frac{C_{13}n\nukpo}{d+\tau}  - \frac{C_{14}n\frac{1}{C}\pts{\ndpnd}\nut\nukpo}{\pts{d+\tau}^2}\Bigg)\Bigg)\\
        & \geq \frac{\pts{\hc_{15}-\frac{1}{C_{16}}}\pts{\ndpnd}}{d+\tau}.
    \end{align*}
    $\bullet \; \stdokpo$: We have
    \begin{align*}
        \stdokpo &= \bv_2^\top\bDelta^{-1}\XXkpotaui\bv_1\\
        &=\bv_2^\top\bDelta^{-1}\XXktaui\bv_1 - \bv_2^\top\bDelta^{-1}\XXktaui\bL_{k+1}\A_{k+1}^{-1}\bR_{k+1}\XXktaui\bv_1\\
        &= \stdok - \mts{\nukpo\stdkpok,~\hkpotdk,~\stdkpok} \times\frac{\text{adj}(\A_{k+1})}{\det(\A_{k+1})}\begin{bmatrix}
        \nukpo\sokpok \\
        \sokpok \\
        \hkpook
        \end{bmatrix}\\
        &= \frac{\stdok \det(\A_{k+1}) - \ffs[\A_{k+1}]{\stdkpok, \hkpotdk, \sokpok, \hkpook}}{\det(\A_{k+1})}.
    \end{align*}
    Thus, we can apply Lemma~\ref{lem:aux_f_func} to get 
    \begin{align*}
        \stdokpo &\leq \frac{\normo{\stdok \det(\A_{k+1})} + \normo{\ffs[\A_{k+1}]{\stdkpok, \hkpotdk, \sokpok, \hkpook}}}{\det(\A_{k+1})}\\
        &\leq \frac{1}{\det(\A_{k+1})}\Bigg(\det(\A_{k+1})\stdok + \nukpo^2\normo{\stdkpok}\normo{\sokpok} + \normo{\skpokpok}\normo{\hkpotdk}\normo{\hkpook}\\
        &\myquad[2]+ \pts{1+\frac{1}{C_1}}\pts{\normo{\stdkpok}\normo{\hkpook} + \normo{\hkpotdk}\normo{\sokpok}}\Bigg)\\
        & \leq \frac{1}{\det(\A_{k+1})} \Bigg(\frac{C_2+1}{C_2}\frac{\hc_3\pts{\ndmnd}+\frac{1}{C_4}\pts{\ndpnd}}{d+\tau}\\
        &\myquad[2] + \frac{C_5+1}{C_5}\frac{\pts{\hc_6+\frac{1}{C_7}}\pts{\ndpnd}}{d+\tau}\frac{\nukpo^2n}{d+\tau} + \frac{C_8+1}{C_8}\frac{C_9\nukpo^2n^2\sqrt{n\nd}}{\pts{d+\tau}^3}\\
        &\myquad[2] + \frac{C_{10}+1}{C_{10}}\Bigg(\frac{\pts{\hc_{11}+\frac{1}{C_{12}}}\pts{\ndpnd}}{d+\tau}\frac{C_{13}n\nukpo}{d+\tau} + \frac{C_{14}n\sqrt{n\nd}\nukpo}{\pts{d+\tau}^2}\Bigg)\Bigg)
    \end{align*}
    where in the third inequality we apply the lower bound on $\det(\A_{k+1})$ (Lemma~\ref{lem:aux_det_bound}) and the $k$-th order primitive bounds. Next, Lemma~\ref{lem:aux_compare_a_h} tells us that $\sqrt{n\nd}\leq \frac{1}{C}\pts{\ndpnd}\nut$. We then have
    \begin{align*}
        \stdokpo &\leq \frac{1}{\det(\A_{k+1})} \Bigg(\frac{C_2+1}{C_2}\frac{\hc_3\pts{\ndmnd}+\frac{1}{C_4}\pts{\ndpnd}}{d+\tau}\\
        &\myquad[2] + \frac{C_5+1}{C_5}\frac{\pts{\hc_6+\frac{1}{C_7}}\pts{\ndpnd}}{d+\tau}\frac{\nukpo^2n}{d+\tau} + \frac{C_8+1}{C_8}\frac{C_9\nukpo^2\nut n^2 \frac{1}{C}\pts{\ndpnd}}{\pts{d+\tau}^3}\\
        &\myquad[2] + \frac{C_{10}+1}{C_{10}}\Bigg(\frac{\pts{\hc_{11}+\frac{1}{C_{12}}}\pts{\ndpnd}}{d+\tau}\frac{C_{13}n\nukpo}{d+\tau} + \frac{C_{14}n\frac{1}{C}\pts{\ndpnd}\nut\nukpo}{\pts{d+\tau}^2}\Bigg)\Bigg)\\
        & \leq \frac{\hc_{15}\pts{\ndmnd}+\frac{1}{C_{16}}\pts{\ndpnd}}{d+\tau},
    \end{align*} 
    where in the last inequality, we apply Assumption~\ref{asm:dataset_gen}(C), i.e.~$d \geq CR_+n$.  Similarly,
    \begin{align*}
        \stdokpo &\geq \frac{\stdok \det(\A_{k+1}) - \normo{\ffs[\A_{k+1}]{\stdkpok, \hkpotdk, \sokpok, \hkpook}}}{\det(\A_{k+1})}\\
        &\geq \frac{1}{\det(\A_{k+1})}\Bigg(\det(\A_{k+1})\stdok -\Bigg(\nukpo^2\normo{\stdkpok}\normo{\sokpok} + \normo{\skpokpok}\normo{\hkpotdk}\normo{\hkpook}\\
        &\myquad[2] + \pts{1+\frac{1}{C_1}}\pts{\normo{\stdkpok}\normo{\hkpook} + \normo{\hkpotdk}\normo{\sokpok}}\Bigg)\Bigg)\\
        & \geq \frac{1}{\det(\A_{k+1})} \Bigg(\frac{C_2-1}{C_2}\frac{\hc_3\pts{\ndmnd}-\frac{1}{C_4}\pts{\ndpnd}}{d+\tau}\\
        &\myquad[2] - \frac{C_5+1}{C_5}\frac{\pts{\hc_6+\frac{1}{C_7}}\pts{\ndpnd}}{d+\tau}\frac{\nukpo^2n}{d+\tau} - \frac{C_8+1}{C_8}\frac{C_9\nukpo^2\nut n^2 \frac{1}{C}\pts{\ndpnd}}{\pts{d+\tau}^3}\\
        &\myquad[2] - \frac{C_{10}+1}{C_{10}}\Bigg(\frac{\pts{\hc_{11}+\frac{1}{C_{12}}}\pts{\ndpnd}}{d+\tau}\frac{C_{13}n\nukpo}{d+\tau} - \frac{C_{14}n\frac{1}{C}\pts{\ndpnd}\nut\nukpo}{\pts{d+\tau}^2}\Bigg)\Bigg)\\
        & \geq \frac{\hc_{15}\pts{\ndmnd}-\frac{1}{C_{16}}\pts{\ndpnd}}{d+\tau}.
    \end{align*}
    $\bullet \; \sididkpo$ for $i = 1,2 $: We have
    \begin{align*}
        \sididkpo &= \bv_i^\top\bDelta^{-1}\XXkpotaui\bDelta^{-1}\bv_i\\
        &=\bv_i^\top\bDelta^{-1}\XXktaui\bDelta^{-1}\bv_i\\
        &\myquad[2]-\bv_i^\top\bDelta^{-1}\XXktaui\bL_{k+1}\A_{k+1}^{-1}\bR_{k+1}\XXktaui\bDelta^{-1}\bv_i\\
        &= \sididk - \mts{\nukpo\sidkpok,~\hkpoidk,~\sidkpok} \times\frac{\text{adj}(\A_{k+1})}{\det(\A_{k+1})}\begin{bmatrix}
        \nukpo\sidkpok \\
        \sidkpok \\
        \hkpoidk
        \end{bmatrix}\\
        &= \frac{\sididk \det(\A_{k+1}) - \ffs[\A_{k+1}]{\sidkpok, \hkpoidk, \sidkpok, \hkpoidk}}{\det(\A_{k+1})}.
    \end{align*}
    Thus, we can apply Lemma~\ref{lem:aux_f_func} to get 
    \begin{align*}
        \sididkpo &\leq \frac{\normo{\sididk \det(\A_{k+1})} + \normo{\ffs[\A_{k+1}]{\sidkpok, \hkpoidk, \sidkpok, \hkpoidk}}}{\det(\A_{k+1})}\\
        &\leq \frac{1}{\det(\A_{k+1})}\Bigg(\det(\A_{k+1})\normo{\sididk} + \nukpo^2\normo{\sidkpok}\normo{\sidkpok} + \normo{\skpokpok}\normo{\hkpoidk}\normo{\hkpoidk}\\
        &\myquad[2]+ \pts{1+\frac{1}{C_1}}\pts{\normo{\sidkpok}\normo{\hkpoidk} + \normo{\hkpoidk}\normo{\sidkpok}}\Bigg)\\
        & \leq \frac{1}{\det(\A_{k+1})} \Bigg(\frac{C_2+1}{C_2}\frac{\hc_3\nd}{d+\tau} + \frac{\pts{\hc_4+\frac{1}{C_5}}\pts{\ndpnd}^2\nukpo^2}{\pts{d+\tau}^2}\\
        &\myquad[2] + \frac{C_6+1}{C_6}\frac{C_7\nukpo^2n^2\nd}{\pts{d+\tau}^3} + \frac{C_{8}+1}{C_{8}}\frac{2\pts{\hc_{9}+\frac{1}{C_{10}}}\pts{\ndpnd}}{d+\tau}\frac{C_{11}\sqrt{n\nd}\nukpo}{d+\tau}\Bigg)
    \end{align*}
    where in the third inequality we apply the lower bound on $\det(\A_{k+1})$ (Lemma~\ref{lem:aux_det_bound}) and the $k$-th order primitive bounds. Next, by the Cauchy–Schwarz inequality, we have $\pts{\ndpnd}^2\leq 2 \pts{\frac{n_+^2}{\Delta_+^2} + \frac{n_-^2}{\Delta_-^2}} \leq 2n\nd$. Also, Lemma~\ref{lem:aux_compare_a_h} tells us that $\sqrt{n\nd}\leq \frac{1}{C}\pts{\ndpnd}\nut$. We then have
    \begin{align*}
        \sididkpo &\leq \frac{1}{\det(\A_{k+1})} \Bigg(\frac{C_2+1}{C_2}\frac{\hc_3\nd}{d+\tau}  + \frac{\pts{\hc_4+\frac{1}{C_5}}2n\nd\nukpo^2}{\pts{d+\tau}^2} + \frac{C_6+1}{C_6}\frac{C_7\nukpo^2 n^2\nd }{\pts{d+\tau}^3}\\
        &\myquad[2]+ \frac{C_{8}+1}{C_{8}}\frac{2\pts{\hc_{9}+\frac{1}{C_{10}}}\frac{2}{C}n\nd\nut}{d+\tau}\frac{C_{11}\nukpo}{d+\tau}\Bigg)\\
        & \leq \frac{\hc_{12}\nd}{d+\tau},
    \end{align*} 
    where in the last inequality, we apply Assumption~\ref{asm:dataset_gen}(C), i.e.~$d \geq CR_+n$.  Similarly,
    \begin{align*}
        \sididkpo &\geq \frac{\sididk \det(\A_{k+1}) - \normo{\ffs[\A_{k+1}]{\sidkpok, \hkpoidk, \sidkpok, \hkpoidk}}}{\det(\A_{k+1})}\\
        &\geq \frac{1}{\det(\A_{k+1})}\Bigg(\det(\A_{k+1})\sididk -\Bigg(\nukpo^2\normo{\sidkpok}\normo{\sidkpok} + \normo{\skpokpok}\normo{\hkpoidk}\normo{\hkpoidk}\\
        &\myquad[2]+ \pts{1+\frac{1}{C_1}}\pts{\normo{\sidkpok}\normo{\hkpoidk} + \normo{\hkpoidk}\normo{\sidkpok}}\Bigg)\Bigg)\\
        & \geq \frac{1}{\det(\A_{k+1})} \Bigg(\frac{C_2+1}{C_2}\frac{\dc_3\nd}{d+\tau} - \frac{\pts{\hc_4+\frac{1}{C_5}}2n\nd\nukpo^2}{\pts{d+\tau}^2}\\
        &\myquad[2] - \frac{C_6+1}{C_6}\frac{C_7\nukpo^2 n^2\nd }{\pts{d+\tau}^3} - \frac{C_{8}+1}{C_{8}}\frac{2\pts{\hc_{9}+\frac{1}{C_{10}}}\frac{2}{C}n\nd\nut}{d+\tau}\frac{C_{11}\nukpo}{d+\tau}\Bigg)\\
        & \geq \frac{\dc_{12}\nd}{d+\tau}.
    \end{align*}
    $\bullet \; \hijdkpo$ for $i, j = 1,2 $: We have
    \begin{align*}
        \hijdkpo &= \bd_i^\top\XXkpotaui\bDelta^{-1}\bv_j\\
        &=\bd_i^\top\XXktaui\bDelta^{-1}\bv_j - \bd_i^\top\XXktaui\bL_{k+1}\A_{k+1}^{-1}\bR_{k+1}\XXktaui\bDelta^{-1}\bv_j\\
        &= \hijdk - \mts{\nukpo\hikpok,~\tikpok,~\hikpok} \times\frac{\text{adj}(\A_{k+1})}{\det(\A_{k+1})}\begin{bmatrix}
        \nukpo\skpojdk \\
        \skpojdk \\
        \hkpojdk
        \end{bmatrix}\\
        &= \frac{\hijdk \det(\A_{k+1})- \ffs[\A_{k+1}]{\hikpok, \tikpok, \skpojdk, \hkpojdk}}{\det(\A_{k+1})}.
    \end{align*}
    Thus, we can apply Lemma~\ref{lem:aux_f_func} to get 
    \begin{align*}
        \hijdkpo &\leq \frac{\normo{\hijdk \det(\A_{k+1})} +\normo{\ffs[\A_{k+1}]{\hikpok, \tikpok, \skpojdk, \hkpojdk}}}{\det(\A_{k+1})}\\
        &\leq \frac{1}{\det(\A_{k+1})}\Bigg(\det(\A_{k+1})\normo{\hijdk} + \nukpo^2\normo{\hikpok}\normo{\skpojdk} + \normo{\skpokpok}\normo{\tikpok}\normo{\hkpojdk}\\
        &\myquad[2]+ \pts{1+\frac{1}{C_1}}\pts{\normo{\hikpok}\normo{\hkpojdk} + \normo{\tikpok}\normo{\skpojdk}}\Bigg)\\
        & \leq \frac{1}{\det(\A_{k+1})} \Bigg(\frac{C_2+1}{C_2}\frac{C_3\sqrt{n\nd}\nui}{d+\tau}\\
        &\myquad[2] + \frac{\pts{\hc_4+\frac{1}{C_5}}\pts{\ndpnd}\nukpo^2 C_6n\nui}{\pts{d+\tau}^2} + \frac{C_7+1}{C_7}\frac{C_8\nukpo^2\nui n^2\sqrt{n\nd}}{\pts{d+\tau}^3}\\
        &\myquad[2] + \frac{C_{9}+1}{C_{9}}\Bigg(\frac{C_{10}n\sqrt{n\nd}\nui\nukpo}{\pts{d+\tau}^2} + \frac{C_{11}\pts{\hc_{12}+\frac{1}{C_{13}}}n\pts{\ndpnd}\nui\nukpo}{\pts{d+\tau}^2}\Bigg)\Bigg),
    \end{align*}
    where in the third inequality we apply the lower bound on $\det(\A_{k+1})$ (Lemma~\ref{lem:aux_det_bound}) and the $k$-th order primitive bounds. Next, by the Cauchy–Schwarz inequality, we have $\pts{\ndpnd}\leq \sqrt{2 \pts{\frac{n_+^2}{\Delta_+^2} + \frac{n_-^2}{\Delta_-^2}}} \leq \sqrt{2n\nd}$. We then have
    \begin{align*}
        \hijdkpo &\leq \frac{1}{\det(\A_{k+1})} \Bigg(\frac{C_2+1}{C_2}\frac{C_3\sqrt{n\nd}\nui}{d+\tau}\\
        &\myquad[2] + \frac{\pts{\hc_4+\frac{1}{C_5}}C_6n\sqrt{2n\nd}\nui\nukpo^2}{\pts{d+\tau}^2} + \frac{C_7+1}{C_7}\frac{C_8\nukpo^2 \nui n^2\sqrt{n\nd} }{\pts{d+\tau}^3}\\
        &\myquad[2] + \frac{C_{9}+1}{C_{9}}\Bigg(\frac{C_{10}n\sqrt{n\nd}\nui\nukpo}{\pts{d+\tau}^2} + \frac{C_{11}\pts{\hc_{12}+\frac{1}{C_{13}}}n\sqrt{2n\nd}\nui\nukpo}{\pts{d+\tau}^2}\Bigg)\Bigg)\\
        & \leq \frac{C_{12}\sqrt{n\nd}\nui}{d+\tau},
    \end{align*} 
    where in the last inequality, we apply Assumption~\ref{asm:dataset_gen}(C), i.e.~$d \geq CR_+n$.  Similarly,
    \begin{align*}
        \hijdkpo &\geq \frac{\hijdk \det(\A_{k+1})- \normo{\ffs[\A_{k+1}]{\hikpok, \tikpok, \skpojdk, \hkpojdk}}}{\det(\A_{k+1})}\\
        &\geq \frac{1}{\det(\A_{k+1})}\Bigg(\det(\A_{k+1})\hijdk -\Bigg(\nukpo^2\normo{\hikpok}\normo{\skpojdk} + \normo{\skpokpok}\normo{\tikpok}\normo{\hkpojdk}\\
        &\myquad[2]+ \pts{1+\frac{1}{C_1}}\pts{\normo{\hikpok}\normo{\hkpojdk} + \normo{\tikpok}\normo{\skpojdk}}\Bigg)\Bigg)\\
        & \geq \frac{1}{\det(\A_{k+1})} \Bigg(-\frac{C_2-1}{C_2}\frac{C_3\sqrt{n\nd}\nui}{d+\tau}\\
        &\myquad[2] - \frac{\pts{\hc_4+\frac{1}{C_5}}C_6n\sqrt{2n\nd}\nui\nukpo^2}{\pts{d+\tau}^2} - \frac{C_7+1}{C_7}\frac{C_8\nukpo^2 \nui n^2\sqrt{n\nd} }{\pts{d+\tau}^3}\\
        &\myquad[2] - \frac{C_{9}+1}{C_{9}}\Bigg(\frac{C_{10}n\sqrt{n\nd}\nui\nukpo}{\pts{d+\tau}^2} - \frac{C_{11}\pts{\hc_{12}+\frac{1}{C_{13}}}n\sqrt{2n\nd}\nui\nukpo}{\pts{d+\tau}^2}\Bigg)\Bigg)\\
        & \geq -\frac{C_{12}\sqrt{n\nd}\nui}{d+\tau}.
    \end{align*}

\section{Proof of Proposition~\ref{pro:wst_lower_bound}} \label{app:lower_bound}
In this section, we provide the proof of Proposition~\ref{pro:wst_lower_bound}. The proof steps mostly resemble those in the proof of Theorem~\ref{thm:wst_risk_upperbound}. We show the derivation for the case $b = +1$, noting that the case $b = -1$ follows by an identical set of calculations.

    First of all, we use the matching exponential lower bound on the Q-function~\cite{chiani2003new}, which yields
    \begin{align} \label{eq:lower_bound_target_term}
        \Riskf[+1]{\hw} = \Qf{\frac{\hw^{\top}\bmu_{+1}}{\sqrt{\hw^{\top}\hw}}} \geq C\expf{-\frac{\pts{\hw^{\top}\bmu_{+1}}^2}{2\hw^{\top}\hw}}.
    \end{align}
    Therefore, it is sufficient to show an upper bound on $\frac{\pts{\hw^{\top}\bmu_{+1}}^2}{2\hw^{\top}\hw}$ to lower bound Eq.~\eqref{eq:lower_bound_target_term}. We first reproduce the equivalent expression in terms of adjusted primitives (Eq.~\eqref{eq:maj_target_primitive}), and then apply the bounds on these primitives (provided in Lemma~\ref{lem:w_primitive_bounds}) to get the following upper bound:
    \begin{align}
        &\frac{\pts{\hw^\top\bmu_{+1}}^2}{2\hw^{\top}\hw}\nonumber\\
        &= \frac{\pts{\nut^2\stdtt + \nuo^2\stdot + \httdt + \hotdt}^2}{2\otdtdt}\nonumber\\
        &\leq \dfrac{\pts{\nut^2 \frac{\pts{\hc_9+\frac{1}{C_{10}}} \pts{\ndpnd}}{d+\tau} + \nuo^2 \frac{\hc_9 \pts{\ndmnd} + \frac{1}{C_{10}}\pts{\ndpnd}}{d+\tau}+\frac{C_{12} \sqrt{n\nd}\nut}{d+\tau}+\frac{C_{12} \sqrt{n\nd}\nuo}{d+\tau}}^2}{\frac{\dc_{13}\nd d}{\pts{d+\tau}^2}}\nonumber\\
        &=\frac{\pts{\pts{\hc_9\nut^2 +\frac{\nut^2+\nuo^2}{C_{10}}}\pts{\ndpnd} + \hc_9 \nuo^2\pts{\ndmnd}{+C_{12} \sqrt{n\nd}\pts{\nut+\nuo}}}^2}{\dc_{13}\nd d}.\label{eq:low_maj_target_primitive2}
    \end{align}
    Next, since $R_- \geq 0$, i.e. $\nut \geq \nuo$, we have
    \begin{align}
        \eqref{eq:low_maj_target_primitive2}&\leq \dfrac{\pts{\pts{\hc_9\nut^2 +\frac{2\nut^2}{C_{10}}}\pts{\ndpnd}+ \hc_9 \nuo^2\pts{\ndmnd}+2C_{12} \sqrt{n\nd}\nut}^2}{\dc_{13}\nd d}\nonumber\\
        &\leq \dfrac{\pts{\pts{\hc_9 +\frac{2}{C_{10}} +\frac{C_{12}}{C_{14}}}\nut^2\pts{\ndpnd}+ \hc_9 \nuo^2\pts{\ndmnd}}^2}{\dc_{13}\nd d}\nonumber\\
        &= \dfrac{\pts{\hc_{15}\nut^2\pts{\ndpnd} + \hc_{15} \nuo^2\pts{\ndmnd}}^2}{\dc_{13}\nd d}\nonumber\\
        &= \dfrac{\pts{\hc_{15} R_+\frac{n_+}{\Delta_+} + \hc_{15} R_-\frac{n_-}{\Delta_-}}^2}{\dc_{13}\nd d} \leq \dfrac{2\hc_{15}^2 R_+^2\frac{n_+^2}{\Delta_+^2} + 2\hc_{15}^2 R_-^2\frac{n_-^2}{\Delta_-^2}}{\dc_{13}\nd d},\label{eq:low_maj_target_primitive3}
    \end{align}
    where the second inequality follows from Lemma~\ref{lem:aux_compare_a_h} applied with a large enough constant $C_{14} > C_{12}$. Finally, we define $\alpha_\pm = \frac{n_\pm/\Delta_\pm^2}{\nd}$, noting that $0<\alpha_\pm<1$ and $\alpha_+ + \alpha_-=1$. Then, Eq.~\eqref{eq:low_maj_target_primitive3} becomes
\begin{align*}
    \frac{\pts{\hw^\top\bmu_{+1}}^2}{2\hw^{\top}\hw} &\leq \frac{C_{16} \pts{\alpha_+R_+^2n_+ + \alpha_-R_-^2n_-}}{d}.
\end{align*}
Plugging the above into Eq.~\eqref{eq:lower_bound_target_term} completes the proof.
\qed
\end{document}